\documentclass{article}

\usepackage[main,final]{neurips_2026}

\usepackage[utf8]{inputenc} % allow utf-8 input
\usepackage[T1]{fontenc}    % use 8-bit T1 fonts
\usepackage{hyperref}       % hyperlinks
\hypersetup{colorlinks=true,linkcolor=black,citecolor=black,urlcolor=blue}
\usepackage{url}            % simple URL typesetting
\usepackage{booktabs}       % professional-quality tables
\usepackage{amsfonts}       % blackboard math symbols
\usepackage{nicefrac}       % compact symbols for 1/2, etc.
\usepackage{microtype}      % microtypography
\usepackage{xcolor}         % colors
\usepackage{algorithm}
\usepackage{algorithmic}
\usepackage{amsmath}
\usepackage{amsthm}
\usepackage{enumitem}

\theoremstyle{plain}
\newtheorem{theorem}{Theorem}[section]
\newtheorem{proposition}[theorem]{Proposition}
\newtheorem{lemma}[theorem]{Lemma}
\newtheorem{corollary}[theorem]{Corollary}
\theoremstyle{definition}

\theoremstyle{remark}

\usepackage{booktabs}
\usepackage{multirow}
\usepackage{makecell}
\usepackage{array}
\usepackage{threeparttable}
\usepackage{adjustbox}
\usepackage[table]{xcolor}
\usepackage{colortbl}
\usepackage{graphicx}

\usepackage{placeins}
\usepackage{wrapfig}

\definecolor{darkblue}{RGB}{31,61,122}
\definecolor{lightgray}{gray}{0.45}

\newcommand{\MainTableStyle}{%
  \footnotesize
  \setlength{\tabcolsep}{8pt}
  \renewcommand{\arraystretch}{1.10}
}
\definecolor{tablegray}{RGB}{235,235,235}
\definecolor{zeroshotgreen}{RGB}{221,230,227}
\definecolor{zeroshotgreenstrong}{RGB}{204,218,214}

\title{MPQ: A Message-Passing View of Post-Training Quantization}

\author{%
  Yanshuo Chen \\ %\thanks{Use footnote for providing further information
  Department of Computer Science\\
  University of Maryland\\
  College Park, MD 20742 \\
  \texttt{cys@umd.edu} \\
  \And
  Chenxi Liu \\
  Department of Computer Science \\
  University of Maryland \\
  College Park, MD 20742 \\
  \texttt{cxliu539@umd.edu} \\
  \AND
  Reza Shirkavand \\
  Department of Computer Science \\
  University of Maryland \\
  College Park, MD 20742 \\
  \texttt{rezashkv@umd.edu} \\
  \And
  Heng Huang \\
  Department of Computer Science \\
  University of Maryland \\
  College Park, MD 20742 \\
  \texttt{heng@umd.edu} \\
}

\begin{document}

\maketitle

\begin{abstract}
  Post-training quantization (PTQ) is a key technique for efficiently deploying large language models, as it compresses weights into low-bit representations without retraining. Existing PTQ solvers, including GPTQ, LDLQ, and block coordinate descent, commit each weight block to a single codeword before its neighbors are considered, leaving near-optimal alternatives unexplored. In this work, we cast row-wise PTQ as posterior inference over discrete codebooks, in which existing solvers correspond to a zero-temperature, hard-decision limit retaining only the top codeword per block. We propose \emph{Message-Passing Quantization} (MPQ), which preserves soft posterior information across blocks during inference and commits to hard codewords only at the output stage. MPQ first uses vector approximate message passing to compute approximate blockwise posteriors, then refines a hard-decision initialization by exact block coordinate descent. Across 2-bit lattice and 3- and 4-bit scalar quantization on dense language models from Llama, Qwen, and Mistral families, MPQ improves perplexity and zero-shot accuracy over strong hard-decision baselines, and composes favorably with preprocessing techniques such as activation-aware scaling and incoherence processing.
Our implementation is available at \url{https://github.com/poseidonchan/MPQ}.

\end{abstract}

\section{Introduction}
\label{sec:introduction}
% Post-training quantization (PTQ) compresses a pretrained language model by mapping full-precision weights to low-bit discrete codes, enabling deployment under tight memory and compute budgets without additional retraining. However, recent progress has concentrated heavily on the front-end, while back-end solvers have evolved more incrementally.
% \reza{what is front end and back end? shouldn't you explain?}
% On the front-end side, activation scaling methods, incoherence processing methods, structured codebooks and inter-layer error propagation objective design have substantially expanded the feasibility of ultra-low-bit quantization \citep{awq,smoothquant,quip,quipsharp,quarot,spinquant,qtip,yaqa}. \reza{separate the previous citations and put them in the right place, e.g. put activation scaling methods after activation scaling.} However, on the back-end solver side, most existing methods quantize blocks sequentially, make an immediate hard assignment to a single codeword, and compensate for the resulting error greedily. GPTQ, LDLQ, and subsequent block coordinate descent variants differ in scheduling and local approximation, but all remain fundamentally sequential hard-decision solvers \citep{gptq,quip,quantease,cdquant,ojbkq}. This design is efficient, but it commits early to one discrete trajectory and leaves nearby high-quality solutions unexplored.

The growing capabilities of large language models have made them increasingly useful, motivating efficient deployment on personal devices~\citep{gpt4v,multi-crit,self-debate,chen2025watermark,mbpo}. Post-training quantization (PTQ) compresses a pretrained language model by mapping full-precision weights to low-bit discrete codes, enabling deployment under tight memory and compute budgets without additional retraining. Recent progress has mainly improved the preprocessing and objective-design stage of PTQ: activation scaling methods reduce quantization difficulty by rescaling weights and activations~\citep{awq,smoothquant,quarot,spinquant}; incoherence processing methods make weights easier to quantize through rotations or related transformations~\citep{quip,quipsharp}; structured codebooks improve the expressiveness of low-bit representations~\citep{qtip,quipsharp}; and inter-layer error propagation objectives better account for how quantization errors affect later layers~\citep{yaqa,qronos}.

In contrast, the discrete solver stage has changed more incrementally. Most existing solvers quantize blocks sequentially, make an immediate hard assignment to a single codeword, and then greedily compensate for the resulting error. GPTQ, LDLQ, and later block coordinate descent variants differ in scheduling and local approximation, but they remain fundamentally hard-decision solvers \citep{gptq,quip,quantease,cdquant,ojbkq}. This approach is efficient but commits to a single hard codeword per block before its neighbors are processed.

% Such early commitment is most consequential in the low-bit regime. When the calibration Hessian strongly couples blocks and the discrete codebook leaves several codewords nearly tied at each step, a sequential hard-decision solver that commits early can foreclose globally favorable trajectories before they are explored. Once preprocessing, representation and objective design have improved conditioning, the main bottleneck has therefore shifted to the solver itself. Rather than collapsing each block to a single codeword, the solver should carry multiple near-optimal candidates forward and defer commitment until subsequent blocks provide additional information.

This early commitment is especially harmful in the low-bit regime. With very few available codes, several codewords can produce nearly the same local error, and strong Hessian coupling makes the best choice for one block depend on its neighbors. A hard-decision solver therefore commits to one codeword trajectory through the discrete assignment space, with no mechanism to recover if those early choices turn out to be globally suboptimal. As preprocessing, codebook, and objective design improve the conditioning of the quantization problem, the discrete solver remains a relatively underexplored axis. Instead of assigning each block to a single codeword immediately, the solver should maintain a soft distribution over plausible codewords for each block and let these distributions interact with their neighbors before any block commits to a hard decision.

% We start from a different angle and frame row-wise PTQ as a Gaussian-discrete inference problem. Under this view, existing solvers such as GPTQ, LDLQ, and block coordinate descent are zero-temperature procedures on the same underlying posterior, differing only in how they traverse the combinatorial search space. This unified view motivates exploring finite-temperature inference, which maintains a distribution over plausible codewords before making any hard commitment. We instantiate this idea with \emph{Message-Passing Quantization} (MPQ), a two-stage solver that moves from finite to zero temperature. It first runs vector approximate message passing (VAMP) at finite temperature to collect global information about near-tied codeword assignments and then projects the VAMP posterior to a hard initialization and refines it with block coordinate descent. Experiments on models from the Llama, Qwen, and Mistral families show that MPQ consistently matches or improves upon greedy solver baselines, with the largest gains at 3-bit scalar and 2-bit lattice precision.

We take a message-passing view of row-wise PTQ, casting it as posterior inference over a discrete codebook with a Gaussian likelihood. From this view, existing solvers such as GPTQ, LDLQ, and block coordinate descent are zero-temperature message-passing schemes that propagate point estimates between blocks and discard other plausible codewords. This perspective naturally suggests a finite-temperature alternative, where messages carry distributions over several likely codewords before any block commits.

We instantiate this idea with \emph{Message-Passing Quantization} (MPQ), a two-stage solver that first explores uncertainty and then commits to a discrete solution. In the first stage, MPQ uses vector approximate message passing (VAMP) at finite temperature to aggregate information about near-tied codeword assignments across blocks. In the second stage, it converts the resulting posterior into a hard initialization and refines it with block coordinate descent. Experiments on Llama, Qwen, and Mistral models show that MPQ matches or improves over greedy solver baselines, with the largest gains at 3-bit scalar and 2-bit lattice precision.

% Our contributions are as follows.
% \begin{itemize}
% [leftmargin=1.35em,topsep=2pt,itemsep=1.5pt,parsep=0pt]
%   \item We cast row-wise PTQ as inference on a Gaussian-discrete inference problem and place GPTQ, LDLQ, and block coordinate descent under a common zero-temperature interpretation.
%   \item \textcolor{red}{We propose MPQ, a message-passing solver that combines finite-temperature VAMP with block coordinate descent, and show that it achieves competitive or superior results across 2/3/4-bit quantization on seven models from three families.} \reza{if competetitive, it should be faster. if not, don't put competitive here.}
%   \item \textcolor{red}{We provide supporting theory that characterizes the asymptotic behavior of the VAMP solver, explains why the finite-temperature warm start improves zero-temperature refinement, and formalizes the mechanism through which preprocessing benefits MPQ more than hard solvers.}
% \end{itemize}

Our contributions are as follows.
\begin{itemize}
[leftmargin=1.35em,topsep=2pt,itemsep=1.5pt,parsep=0pt]
  \item We give a message-passing view of row-wise PTQ as posterior inference over a discrete codebook under a Gaussian likelihood, recovering existing solvers as zero-temperature limits and motivating finite-temperature inference as an unexplored solver-side direction.
  \item We propose Message-Passing Quantization (MPQ), a two-stage solver that combines finite-temperature vector approximate message passing with exact block coordinate descent.
  \item MPQ improves over strong solver baselines in perplexity and zero-shot accuracy across Llama, Qwen, and Mistral at 2-, 3-, and 4-bit quantization, with the largest gains in the lowest-bit regimes
  \item We complement the method with theoretical analysis of the VAMP solver's asymptotic behavior and the role of preprocessing in finite-temperature inference.
\end{itemize}

\section{Related Work}
\label{sec:relatedwork}
\paragraph{Preprocessing, codebook, and objective design.}
Many studies improve the quantization problem before the solver runs. Scaling and equalization methods such as AWQ, SmoothQuant, OmniQuant, and MagR reduce dynamic-range imbalance or protect salient channels before rounding~\citep{awq,smoothquant,omniquant,magr}. Rotation and incoherence methods such as QuIP, QuIP\#, QuaRot, SpinQuant, and QTIP reshape the Hessian geometry using random or learned orthogonal transforms~\citep{quip,quipsharp,quarot,spinquant,qtip}. Fine-tuning methods such as PV-Tuning further improve the compressed representation after quantization~\citep{pv-tuning}. YAQA and Qronos extend layer-wise PTQ to account for errors propagating across layers~\citep{yaqa,qronos}.

\paragraph{Solver-side PTQ.}
GPTQ made layerwise second-order quantization practical at LLM scale by developing OBQ-style formulations into an efficient column-wise pass~\citep{gptq}. QuIP and QuIP\# extended this to vector quantization with LDLQ and incoherence processing~\citep{quip,quipsharp}. QuantEase and CDQuant revisit local repair via explicit block coordinate descent~\citep{quantease,cdquant}. Recent theoretical work has shown that GPTQ is equivalent to Babai's nearest-plane algorithm on the quantization lattice~\citep{chen2025geometry,birnick2025lattice}. Building on this lattice perspective, OJBKQ further augments greedy Babai rounding with Klein's sampling to sample multiple stochastic trajectories in the search space~\citep{ojbkq}. While Klein sampling injects stochasticity across multiple Babai trajectories, each individual trajectory still commits to one codeword per block sequentially. Section~\ref{sec:preliminary} unifies these solvers as zero-temperature limits of a common posterior view.

% \paragraph{Message passing for discrete inference.}
% Row-wise PTQ is a linear inverse problem with discrete constraints. Methods in the approximate message passing (AMP) family are designed for such problems and maintain soft beliefs over each block rather than committing to a single codeword~\citep{amp,bayati2011dynamics,gamp}. VAMP further extends AMP to a broader operator class with state-evolution guarantees under a right-orthogonal invariance assumption, which the random Hadamard transform (RHT) common in modern PTQ pipelines naturally satisfies~\citep{vamp}. We therefore adopt VAMP as the finite-temperature inference stage of MPQ, with RHT preprocessing both improving conditioning and providing the structural assumption VAMP requires.

\paragraph{Message passing for discrete inference.}
Row wise PTQ is a linear inverse problem with discrete constraints. Methods in the approximate message passing family maintain soft beliefs over latent blocks rather than collapsing immediately to point estimates \citep{amp,bayati2011dynamics,gamp}. VAMP extends AMP with rigorous state evolution for large right orthogonally invariant operators \citep{vamp}. Although the random Hadamard transform used in modern PTQ pipelines does not make a finite calibration Hessian exactly right orthogonally invariant, it randomizes the coordinate basis and improves incoherence, making the right orthogonally invariant model a useful surrogate for analysis \citep{quip,quipsharp}. We therefore use VAMP as the finite temperature inference stage of MPQ and interpret the associated state evolution as qualitative guidance for the RHT conditioned PTQ problem.

\section{Preliminaries}
\label{sec:preliminary}
\paragraph{PTQ formulation.}
Let $W\in\mathbb{R}^{m_{\mathrm{out}}\times n}$ be the weight matrix of a single linear layer and $X\in\mathbb{R}^{m\times n}$ a matrix of calibration activations. Weight-only post-training quantization replaces $W$ with a quantized matrix $\hat W$ by minimizing the layer-wise reconstruction objective $\|X(W-\hat W)^\top\|_F^2$, which decomposes over the $m_{\mathrm{out}}$ output rows.

We identify each output row of $W$ with its transpose, so a single row is written as a column vector $w\in\mathbb{R}^n$. We partition the $n$ coordinates into $B=n/d$ non-overlapping blocks of dimension~$d$, assuming $d$ divides $n$ (or after grouping/padding otherwise). Each block~$b$ is constrained to a finite codebook $\mathcal{C}_b\subset\mathbb{R}^d$ of size~$M_b$; this covers both uniform scalar grids ($d=1$) and structured vector codebooks such as E8P12 ($d=8$, $M_b=2^{16}$)~\citep{quipsharp}. Writing $\hat w=(\hat w_1,\dots,\hat w_B)$ with $\hat w_b\in\mathcal{C}_b$, the per-row problem is
\begin{equation}
  \min_{\hat w_1\in\mathcal{C}_1,\;\dots,\;\hat w_B\in\mathcal{C}_B}\;
  \frac{1}{2}\bigl\|Xw-X\hat w\bigr\|_2^2.
  \label{eq:row_obj}
\end{equation}
Because the quadratic form $\frac{1}{2}\hat w^\top (X^\top X)\hat w$ couples all blocks through the off-diagonal structure of $X^\top X$, changing one block perturbs the residual seen by every other, making it a combinatorial problem over $\prod_b M_b$ configurations.

\paragraph{A unified probabilistic view.}
Fix an output row $w=(w_1,\ldots,w_B)$, and let
$E_w(\hat w):=\frac12\|Xw-X\hat w\|_2^2$ denote its per-row
contribution to the layer-wise reconstruction objective on the product
codebook $\mathcal C_{\mathrm{prod}}=\mathcal C_1\times\cdots\times\mathcal C_B$.
We introduce the finite-temperature Gibbs measure
\begin{equation}
\begin{aligned}
p_\beta(\hat w\mid w,X)
&=
\frac{1}{Z_\beta}
\exp\{-\beta E_w(\hat w)\}
\prod_{b=1}^B \rho_b(\hat w_b), \\
\rho_b(c)
&\propto
\exp\!\left\{
-\frac{\|c-w_b\|_2^2}{2\tau^2\sigma_b^2}
\right\},
\qquad c\in\mathcal C_b .
\end{aligned}
\label{eq:posterior}
\end{equation}
Equivalently, this Gibbs measure is the posterior induced by the
Gaussian likelihood
$Xw\mid \hat w\sim\mathcal N(X\hat w,\beta^{-1}I_m)$ and the
codebook-supported reference measure $\prod_{b=1}^B\rho_b$, where
$\beta$ is the inverse temperature. The reference measure has anchoring width $\tau\sigma_b$ formed from a fixed per-block scale $\sigma_b$ and a global multiplier $\tau$ that broadens the prior uniformly across blocks. %The unanchored case $\tau=\infty$ recovers the uniform reference measure.

Since $\mathcal C_{\mathrm{prod}}$ is finite and $\rho_b(c)>0$ for all
$c\in\mathcal C_b$, the zero-temperature limit $\beta\to\infty$
concentrates on the minimizer set of Eq.~\eqref{eq:row_obj}. Thus, in the zero-temperature limit, MAP inference under this Gibbs measure recovers the minimizers of the original per-row PTQ problem in Eq.~\eqref{eq:row_obj}. Hard-assignment PTQ solvers such as GPTQ, LDLQ, and block coordinate descent can be viewed as approximations to this zero-temperature MAP problem, differing in how they schedule blocks and how they account for inter-block coupling. At finite $\beta$, the posterior retains blockwise means, variances, and probability mass on near-best codewords, which our message-passing solver propagates.

\section{Message-Passing Quantization}
\label{sec:method}
MPQ instantiates finite-temperature inference on the Gibbs posterior in Eq.~\eqref{eq:posterior} and then refines the resulting hard assignment by exact block coordinate descent on the row-wise PTQ objective in Eq.~\eqref{eq:row_obj}. The first stage preserves uncertainty over near-tied codewords before commitment, and the second stage performs exact local repair after projection to a discrete initialization.

\begin{figure}[t]
\centering
\includegraphics[width=0.94\linewidth]{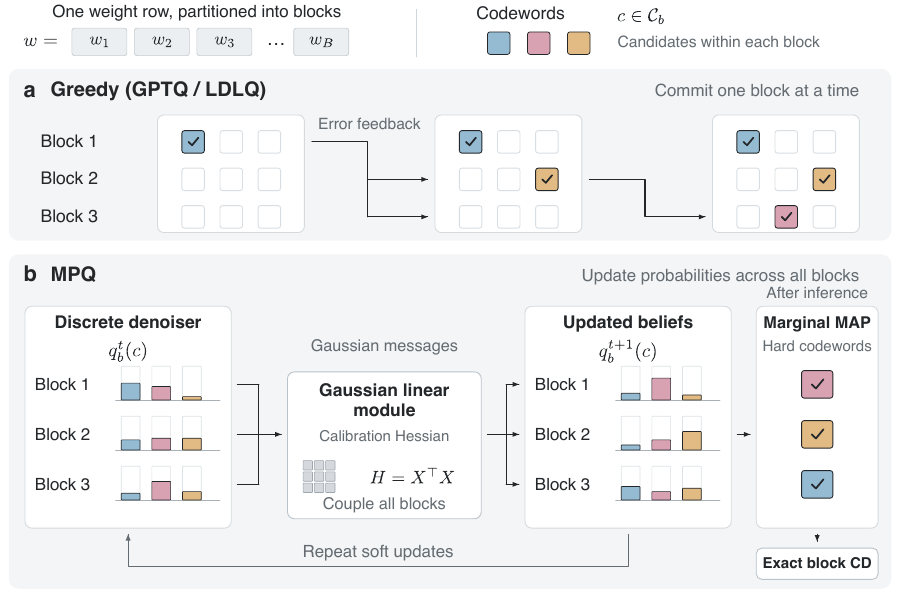}
\caption{{Greedy quantization and MPQ. Rows represent weight blocks $b$. Colors identify codewords $c\in\mathcal C_b$, and bar heights show probabilities $q_b^t(c)$ (Eq.~\eqref{eq:denoiser_posterior}). (a) Greedy solvers fix blocks sequentially and pass error to remaining blocks. (b) MPQ alternates denoising and Gaussian linear updates across all blocks, then applies marginal MAP projection and exact block CD.}}
\label{fig:mpq_overview}
\end{figure}

\subsection{Finite-temperature VAMP in rank space}
\label{sec:vamp_rank_space}

Let $A:=X$ and $y:=Xw$ with $x\in\mathbb R^n$ denoting the quantized row $\hat w$ in VAMP notation, so that the row-wise problem (Eq.~\eqref{eq:row_obj}) becomes $\min_{x\in\mathcal C_{\mathrm{prod}}}\|y-Ax\|_2^2$. Because the objective depends on $x$ only through the rank-$r$ coordinates $\Sigma_A V^\top x$, the first stage can be implemented entirely in rank space. With the compact SVD $A=U\Sigma_A V^\top$, where $\Sigma_A=\operatorname{diag}(s_1,\ldots,s_r)$ and $r=\operatorname{rank}(A)$, set $T:=\Sigma_A V^\top$ and $\tilde y:=U^\top y$, so that
\begin{equation}
\|y-Ax\|_2^2=\|\tilde y-Tx\|_2^2+\mathrm{const},
\label{eq:rank_space_obj}
\end{equation}
and all first-stage computations use $T\in\mathbb R^{r\times n}$ rather than the original calibration matrix.

At iteration $t$, VAMP alternates between a blockwise discrete denoiser and a Gaussian linear module. Given a Gaussian cavity $(r_1^t,\gamma_1^t)$, with $r_{1,b}^t\in\mathbb R^d$ denoting the $b$-th block of $r_1^t$, the denoiser updates each block independently via
\begin{equation}
q_b^t(c)
\propto
\rho_b(c)\exp\!\left(-\frac{\gamma_1^t}{2}\|c-r_{1,b}^t\|_2^2\right),
\qquad c\in \mathcal C_b,
\label{eq:denoiser_posterior}
\end{equation}
with blockwise posterior mean and covariance
\begin{equation}
\hat x_{1,b}^t
=\sum_{c\in \mathcal C_b} c\,q_b^t(c),
\qquad
\Sigma_{1,b}^t
=\sum_{c\in \mathcal C_b}(c-\hat x_{1,b}^t)(c-\hat x_{1,b}^t)^\top q_b^t(c),
\label{eq:denoiser_moments}
\end{equation}
and average sensitivity
\begin{equation}
\alpha_1^t=\frac{\gamma_1^t}{n}\sum_{b=1}^B \operatorname{tr}\!\left(\Sigma_{1,b}^t\right).
\label{eq:alpha1}
\end{equation}
The standard VAMP extrinsic update then converts $(\hat x_1^t,\alpha_1^t)$ into the Gaussian message $(r_2^t,\gamma_2^t)$ for the linear module; exact formulas are deferred to Appendix~\ref{app:impl_extrinsic}.

With $H:=A^\top A$ and $g:=A^\top y$, the linear module solves the Gaussian quadratic subproblem $\hat x_2^t=(\beta^t H+\gamma_2^t I)^{-1}(\beta^t g+\gamma_2^t r_2^t)$ at the current empirical-Bayes inverse temperature $\beta^t$ (adapted within iterations, see Appendix~\ref{app:impl_extrinsic}), which in rank space evaluates to
\begin{equation}
\hat x_2^t
=
r_2^t
+
V\operatorname{diag}\!\left(\left\{\frac{\beta^t s_\ell}{\beta^t s_\ell^2+\gamma_2^t}\right\}_{\ell=1}^r\right)
\left(\tilde y-\Sigma_A V^\top r_2^t\right),
\label{eq:linear_rank_space}
\end{equation}
with average sensitivity
\begin{equation}
\alpha_2^t=\frac{1}{n}\operatorname{tr}\!\left[\gamma_2^t(\beta^t H+\gamma_2^t I)^{-1}\right].
\label{eq:alpha2}
\end{equation}
A second extrinsic update then maps $(\hat x_2^t,\alpha_2^t)$ back to $(r_1^{t+1},\gamma_1^{t+1})$, again with exact formulas in Appendix~\ref{app:impl_extrinsic}. In the SVD form, the null-space component of $r_2^t$ is left unchanged, corresponding to the $(n-r)/n$ contribution implicit in Eq.~\eqref{eq:alpha2} when $r<n$. We repeat these denoiser and linear-module updates for $T_{\mathrm{vamp}}$ iterations, keeping every block soft-valued throughout the first stage.

\subsection{Zero-temperature exact local refinement}

After $T_{\mathrm{vamp}}$ iterations, MPQ forms one final denoiser posterior at $(r_1^{T_{\mathrm{vamp}}},\gamma_1^{T_{\mathrm{vamp}}})$ and projects it to a hard initialization by blockwise marginal MAP,
\begin{equation}
x_b^{(0)}=\arg\max_{c\in \mathcal C_b} q_b^{T_{\mathrm{vamp}}}(c).
\label{eq:map_init}
\end{equation}

MPQ uses blockwise marginal modes, which need not form a joint MAP assignment because the Hessian couples the blocks. It then runs exact block coordinate descent (CD) on the original PTQ objective in Eq.~\eqref{eq:row_obj}, for up to $S_{\mathrm{CD}}$ sweeps with early stopping. Let
\begin{equation}
u:=\tilde y-Tx
\label{eq:residual_def}
\end{equation}
be the current residual, and let $T_b$ denote the columns of $T$ associated with block $b$. When block $b$ is set to a candidate $c\in \mathcal C_b$, the resulting objective value is
\begin{equation}
F_b(c)=\frac12\|u-T_b(c-x_b)\|_2^2.
\label{eq:block_obj}
\end{equation}
The block update is therefore
\begin{equation}
x_b\leftarrow\arg\min_{c\in \mathcal C_b} F_b(c).
\label{eq:block_update}
\end{equation}
This search is exact within each block, but it is still local because each update is evaluated against the current residual. Appendix~\ref{app:impl_mahalanobis} formalizes this as nearest-codeword search in a Mahalanobis metric whose center moves with the iteration. The VAMP warm start is therefore useful not because it replaces local search, but because it enters refinement from a better basin.

Algorithm~\ref{alg:mpq} summarizes the row-wise MPQ procedure at the level of the mathematical updates. A Python-style pseudocode implementation is given in Appendix~\ref{app:impl_pseudocode}.

\begin{algorithm}[t]
\caption{MPQ for one output row}
\label{alg:mpq}
\small
\begin{algorithmic}[1]
\REQUIRE rank-space $(T,\tilde y)$, row $w$, codebooks $\{\mathcal C_b\}$, prior scales $\{\sigma_b\}$, anchoring multiplier $\tau$, VAMP iterations $T_{\mathrm{vamp}}$, CD budget $S_{\mathrm{CD}}$
\STATE Initialize ($r_1^0, \gamma_1^0, \beta^0$) as in Appendix~\ref{app:impl_extrinsic}
\FOR{$t=0,\ldots,T_{\mathrm{vamp}}-1$}
    \STATE \textbf{Denoiser:} for each block $b$, form $q_b^t$ \eqref{eq:denoiser_posterior} and compute $(\hat x_{1,b}^t,\Sigma_{1,b}^t)$ \eqref{eq:denoiser_moments}; aggregate $\alpha_1^t$ \eqref{eq:alpha1}
    \STATE \textbf{Extrinsic:} apply \eqref{eq:extrinsic_12} to form $(r_2^t,\gamma_2^t)$
    \STATE \textbf{Linear:} compute $\hat x_2^t$ \eqref{eq:linear_rank_space} and $\alpha_2^t$ \eqref{eq:alpha2}; update $\beta^{t+1}$ \eqref{eq:beta_em}
    \STATE \textbf{Extrinsic:} apply \eqref{eq:extrinsic_21} to form $(r_1^{t+1},\gamma_1^{t+1})$
\ENDFOR
\STATE \textbf{MAP projection:} for each block $b$, set $x_b\leftarrow\arg\max_{c\in\mathcal C_b}q_b^{T_{\mathrm{vamp}}}(c)$
\STATE Set $u\leftarrow\tilde y-Tx$
\FOR{$p=1,\ldots,S_{\mathrm{CD}}$}
    \STATE For each $b$, set $x_b^{\mathrm{old}}\leftarrow x_b$, update $x_b$ by \eqref{eq:block_update}, and update $u\leftarrow u-T_b(x_b-x_b^{\mathrm{old}})$
    \STATE \textbf{break} if no block changed during this sweep
\ENDFOR
\RETURN $(x_1,\ldots,x_B)$
\end{algorithmic}
\end{algorithm}

\section{Theoretical Analysis}
\label{sec:theory}
% The key insight of MPQ is that VAMP keeps near-tie information that hard solvers discard. We analyze this effect in an idealized right-orthogonally invariant (ROI) surrogate and track the denoiser-side fixed-point precision $\gamma_1^\star$. Theorem~\ref{thm:warmstart} bounds the MAP projection error probability by an exponential in $\gamma_1^\star$, and Theorem~\ref{thm:linear_response} derives an exact formula for how diagonal preprocessing changes the linear-module error $E_{\mathrm{lin}}$, which through state evolution raises $\gamma_1^\star$. The two results together trace a chain from preprocessing to MAP accuracy.

The key insight of MPQ is that VAMP keeps near-tie information that hard solvers discard. We analyze this effect in an idealized ROI surrogate at the unanchored $\tau\to\infty$ limit, tracking the denoiser-side fixed-point precision $\gamma_1^\star$. Theorem~\ref{thm:warmstart} bounds the MAP projection error probability by an exponential in $\gamma_1^\star$, and Proposition~\ref{thm:linear_response} derives an exact formula for how diagonal preprocessing changes the linear-module error $E_{\mathrm{lin}}$, which through state evolution affects $\gamma_1^\star$. The two results together trace a chain from preprocessing to MAP accuracy. MPQ itself runs at finite $\tau$, so we treat this analysis as qualitative guidance rather than a tight characterization.

\subsection{Posterior concentration under VAMP}
\label{sec:posterior_concentration_under_vamp_anchor}
Consider the blockwise surrogate model
\begin{equation}
y = A x^\star + \xi,
\qquad
\xi \sim \mathcal N(0,\beta_0^{-1} I),
\label{eq:surrogate_main}
\end{equation}
where $y$ is the noisy observation, $x^\star=[C_1^\star;\ldots;C_B^\star]$, $C_b^\star\stackrel{\mathrm{i.i.d.}}{\sim}\operatorname{Unif}(\mathcal C)$ for a finite codebook $\mathcal C\subset\mathbb R^d$, and $A$ is ROI. Under VAMP state evolution, let $(\gamma_1^\star,\gamma_2^\star)$ be a locally stable fixed point reached by damped VAMP \citep[Theorem~1]{vamp}; Appendix~\ref{app:proof_se_specialization} specializes the recursion to our blockwise discrete denoiser. At this fixed point, each block sees an effective additive white Gaussian noise (AWGN) channel of precision $\gamma_1^\star$, and the following theorem bounds the per-block MAP error from that channel.
\begin{theorem}[Posterior concentration and warm-start accuracy]
\label{thm:warmstart}
At the state-evolution fixed point, for any block, the effective channel can be written as
\begin{equation}
R = C^\star + \eta,
\qquad
\eta\sim\mathcal N\bigl(0,(\gamma_1^\star)^{-1}I_d\bigr),
\label{eq:effective_channel_two}
\end{equation}
where $C^\star\sim\operatorname{Unif}(\mathcal{C})$ is the true codeword, indexed as $C^\star = c^{(j^\star)}$. For each $j$, define the score
\begin{equation*}
s_j(R):=-\frac{\gamma_1^\star}{2}\|R-c^{(j)}\|_2^2,
\end{equation*}
and define $\delta_j := c^{(j)} - c^{(j^\star)}$ for each competitor $j \neq j^\star$.
Then the blockwise MAP error $\hat j_{\mathrm{MAP}}:=\arg\max_j s_j(R)$ obeys
\begin{equation}
\mathbb P\bigl(\hat j_{\mathrm{MAP}}\neq j^\star\bigr)
\leq
\sum_{j\neq j^\star}
\exp\!\left(-\frac{\gamma_1^\star\|\delta_j\|_2^2}{8}\right).
\label{eq:map_bound_two}
\end{equation}

\end{theorem}

\begin{proof}[Proof sketch.]
For each competitor $j$, the score difference $s_j(R)-s_{j^\star}(R)$ is Gaussian with mean $-\tfrac{\gamma_1^\star}{2}\|\delta_j\|_2^2<0$ and variance $\gamma_1^\star\|\delta_j\|_2^2$, so the standard Gaussian tail bound controls the probability that this competitor outranks the true codeword. A union bound over competitors yields Eq.~\eqref{eq:map_bound_two}. Appendix~\ref{app:proof_warmstart} contains the full argument.

% Hence, for fixed codeword geometry, the MAP warm-start error decays exponentially in the effective precision $\gamma_1^\star$. Under the i.i.d.\ block surrogate, summing the per-block error probability over the $B$ blocks gives the conditional expected number of incorrect blocks, and averaging this per-block probability over the uniform draw of $C^\star$ yields the unconditional bound.

\end{proof}

\paragraph{Remark.}
The bound~\eqref{eq:map_bound_two} makes the intuition developed in Section~\ref{sec:introduction} more precise. When codewords are close in Euclidean distance, the pairwise gaps $\Delta_j$ are small and the bound is loose. As $\gamma_1^\star$ increases, whether through better conditioning or preprocessing, the exponential tails shrink and the bound certifies a reliable MAP projection. Conversely, when the codebook is dense and many nearby competitors remain, the bound is 
no longer informative, and any residual error in this regime is left to the block coordinate descent. Section~\ref{sec:mechanism} reports diagnostics consistent with this prediction, showing that the dense E8P12 codebook yields a smaller terminal precision proxy and a substantially larger fraction of blocks updated by CD than the scalar codebook.

\subsection{Preprocessing and the linear module}
Theorem~\ref{thm:warmstart} shows that a higher effective precision $\gamma_1^\star$ yields a sharper warm start. Through state evolution (Appendix~\ref{app:proof_se_specialization}), $\gamma_1^\star$ is determined jointly by the denoiser MSE on the uniform codebook prior and the linear-module error
\begin{equation}
E_{\mathrm{lin}}(H,\gamma) := \frac{1}{n}\,\operatorname{tr}\bigl[(\beta_0 H+\gamma I)^{-1}\bigr],
\label{eq:Elin_def_main}
\end{equation}
which depends on the Hessian spectrum ( $\beta_0$ is introduced in Eq.~\eqref{eq:surrogate_main}). In this surrogate, fixing the codebook geometry, preprocessing enters the message-passing linear module entirely through $E_{\mathrm{lin}}$. We characterize this dependence for diagonal preprocessing such as AWQ~\citep{awq}.

Let $H_0\succeq 0$ be a fixed baseline Hessian and let $D=\operatorname{diag}(a_1,\dots,a_n)\succ 0$ be a positive diagonal scaling. We interpolate from $H_0$ to $DH_0D$ along the path $H_t := D_t H_0 D_t$ with $D_t := \exp(tL)$ and $L := \log D$, so that $D_0 = I$ and $D_1 = D$.

\begin{proposition}[Preprocessing effect on the linear module]
\label{thm:linear_response}
For every $\gamma>0$, the path $t\mapsto E_{\mathrm{lin}}(H_t,\gamma)$ is continuously differentiable and satisfies
\begin{equation}
\frac{d}{dt}E_{\mathrm{lin}}(H_t,\gamma)
=
-\frac{2\beta_0}{n}\,\operatorname{tr}\!\left[L\,H_t\,(\beta_0 H_t+\gamma I)^{-2}\right].
\label{eq:Elin_derivative_two}
\end{equation}
Consequently,
\begin{equation}
E_{\mathrm{lin}}(H_1,\gamma)-E_{\mathrm{lin}}(H_0,\gamma)
=
-\frac{2\beta_0}{n}\int_0^1 \operatorname{tr}\!\left[L\,H_t\,(\beta_0 H_t+\gamma I)^{-2}\right]dt.
\label{eq:Elin_integral_two}
\end{equation}
\end{proposition}

\begin{proof}[Proof sketch.]
Differentiate $H_t=D_tH_0D_t$ to obtain $\dot H_t=L H_t+H_tL$. Then differentiate the resolvent $(\beta_0 H_t+\gamma I)^{-1}$ with the inverse-derivative identity and use that $H_t$ commutes with any polynomial in itself, in particular with $(\beta_0 H_t+\gamma I)^{-1}$. Cyclicity of trace yields Eq.~\eqref{eq:Elin_derivative_two}, and integrating from $t=0$ to $t=1$ gives Eq.~\eqref{eq:Elin_integral_two}. Appendix~\ref{app:proof_linear_response} contains the full derivation.
\end{proof}

\paragraph{Remark.}
The integral~\eqref{eq:Elin_integral_two} captures the effect of AWQ in isolation, but the production pipeline composes AWQ with a random Hadamard transformation (RHT) $O$, replacing $H_t$ by the conjugated Hessian $O^\top H_t O$. Since $E_{\mathrm{lin}}$ depends on the Hessian only through its eigenvalues, this conjugation leaves $E_{\mathrm{lin}}$ unchanged (Lemma~\ref{lem:orthogonal_invariance_elin}), and the analysis of $\Delta E_{\mathrm{lin}}$ along the AWQ path applies verbatim to the AWQ+RHT pipeline. For clear effect isolation, Corollary~\ref{cor:coordinate_response_elin} rewrites the integral as a sum over coordinates of nonnegative resolvent-weighted Hessian factors $\kappa_i$ scaled by the log-scales $\ell_i$, making the sign explicit. Coordinates with $\ell_i > 0$ decrease $E_{\mathrm{lin}}$ while those with $\ell_i < 0$ increase it, with the magnitude of each contribution weighted by $\kappa_i$. The diagnostic in Section~\ref{sec:mechanism} evaluates this signed quantity on the realized AWQ scales and shows that the net contribution is typically negative, consistent with AWQ lowering $E_{\mathrm{lin}}$ in the preprocessing pipeline.

\section{Experiments}
\label{sec:experiments}

\providecommand{\tg}[1]{\cellcolor{tablegray}#1}
\providecommand{\zg}[1]{\cellcolor{zeroshotgreen}#1}
\providecommand{\zgs}[1]{\cellcolor{zeroshotgreenstrong}#1}

\subsection{Setup}
\label{sec:setup}
\paragraph{Protocol.} 
We use 128 C4~\citep{c4} training-set sequences of length 2048 as calibration data to quantize dense language model weights, while keeping activations at 16 bits during evaluation. We measure WikiText-2~\citep{wikitext} and C4 perplexity and zero-shot normalized accuracy on ARC-Challenge, ARC-Easy~\citep{arc}, HellaSwag~\citep{hellaswag}, PIQA~\citep{piqa}, and Winogrande~\citep{winogrande} using \texttt{lm-eval} harness~\citep{lm-eval}. We also repeat the 3-bit and 4-bit scalar experiments on four Llama models with three independently sampled C4 calibration sets. Methods share each calibration draw. Appendix~\ref{app:calibration_resampling} reports mean perplexity and sample standard deviation.

\paragraph{Codebook and solvers.}
For 3/4-bit scalar codebook quantization, we use a group size of 128 and apply the activation-aware scaling~\citep{awq} followed by RHT~\citep{quip} as preprocessing for all solvers. For the GPTQ baseline, we use the strongest variant from GPTQ~\citep{gptq} with activation-ordered column processing. 
For 2-bit quantization,  we adopt vector quantization with the E8P12 codebook~\citep{quipsharp}, which provides better rate-distortion than scalar codebooks at this aggressive bit width. Following the standard recipe in QuIP\#~\citep{quipsharp}, we employ RHT for incoherence preprocessing, and take LDLQ as the corresponding greedy baseline~\citep{quipsharp}. 
For each greedy solver, GPTQ in the scalar setting and LDLQ in the lattice
setting, we also include a stronger baseline that applies 3 passes of block
coordinate descent (CD), matching the CD refinement used by MPQ. Additionally, we include QuantEase~\citep{quantease} and OJBKQ~\citep{ojbkq} as strong CD and randomized greedy baselines, respectively. Our MPQ solver's default setting is 16 VAMP iterations followed by 3 CD passes. Appendix~\ref{app:pilot} gives the detailed hyperparameter settings for MPQ.

\subsection{Llama-family perplexity scaling}
\label{sec:llama-ppl}

We first compare MPQ with baselines inside the Llama family across different model sizes~\citep{llama2,llama3}. Table~\ref{tab:ptq_ppl_main} shows that MPQ improves over the strongest baseline OJBKQ with the magnitude of the gain governed by the bit budget. At 4 bits, the baselines are already close to the BF16 ceiling and MPQ's lead is small but consistent. The gap widens at 3-bit scalar and widens further at 2-bit E8P12, where the greedy baselines degrade sharply. The GPTQ plus CD and LDLQ plus CD rows isolate the effect of the polishing stage. MPQ remains ahead when all initializations receive the same three CD passes. In the 2-bit lattice regime, its hard projection can start with a higher objective than the greedy initialization and reach a lower objective after the same CD refinement. The absolute gap narrows with model scale, consistent with larger models being more robust to quantization noise, but the ranking is preserved across all model sizes and three precision budgets.

\begin{table*}[h]
\centering
\caption{Llama-family PTQ perplexity scaling.}
\label{tab:ptq_ppl_main}
\begin{threeparttable}
\MainTableStyle
\begin{adjustbox}{width=0.9\textwidth}
\begin{tabular}{l *{5}{cc}}
\toprule
\multirow{2}{*}{\textbf{Method}}
& \multicolumn{2}{c}{\textbf{Llama3.2-1B}}
& \multicolumn{2}{c}{\textbf{Llama3.2-3B}}
& \multicolumn{2}{c}{\textbf{Llama3-8B}}
& \multicolumn{2}{c}{\textbf{Llama2-13B}}
& \multicolumn{2}{c}{\textbf{Llama2-70B}} \\
\cmidrule(lr){2-3}\cmidrule(lr){4-5}\cmidrule(lr){6-7}\cmidrule(lr){8-9}\cmidrule(lr){10-11}
& \textbf{Wiki} & \textbf{C4}
& \textbf{Wiki} & \textbf{C4}
& \textbf{Wiki} & \textbf{C4}
& \textbf{Wiki} & \textbf{C4}
& \textbf{Wiki} & \textbf{C4} \\
\midrule
\rowcolor{tablegray}
BF16
& 9.75 & 12.77
& 7.82 & 10.47
& 6.14 & 8.93
& 4.89 & 6.47
& 3.32 & 5.52 \\
\midrule
\multicolumn{11}{@{}l}{\emph{4-bit scalar codebook, group size 128}} \\
GPTQ
& 10.63 & 13.86
& 8.22 & 11.08
& 6.54 & 9.41
& 4.99 & 6.57
& 3.43 & 5.58 \\
GPTQ + CD
& 10.57 & 13.83
& 8.24 & 11.10
& 6.55 & 9.42
& 4.98 & 6.57
& 3.43 & 5.58 \\
QuantEase & 10.47 & 13.80 & 8.18 & 11.03 & 6.50 & 9.37 & 4.98 & 6.57 & 3.42 & 5.58 \\
OJBKQ & 10.43 & 13.68 & 8.15 & 10.99 & 6.48 & 9.35 & 4.97 & 6.56 & 3.42 & 5.58  \\
\rowcolor{zeroshotgreenstrong}
MPQ
& \textbf{10.38} & \textbf{13.63}
& \textbf{8.12} & \textbf{10.94}
& \textbf{6.46} & \textbf{9.33}
& \textbf{4.96} & \textbf{6.55}
& \textbf{3.41} & \textbf{5.57} \\
\midrule
\multicolumn{11}{@{}l}{\emph{3-bit scalar codebook,  group size 128}} \\
GPTQ
& 15.45 & 19.25
& 10.56 & 13.76
& 8.29 & 11.61
& 5.46 & 7.09
& 3.93 & 5.89 \\
GPTQ + CD
& 15.58 & 18.94
& 10.51 & 13.82
& 8.43 & 11.63
& 5.46 & 7.09
& 3.92 & 5.90 \\
QuantEase & 14.86 & 18.51 & 10.05 & 13.36 & 8.10 & 11.24 & 5.39 & 6.95 & 3.88 & 5.88 \\
OJBKQ & 14.26 & 18.00 & 9.93 & 13.22 & 7.95 & 11.08 & 5.40 & 6.98 & 3.87 & 5.86 \\
\rowcolor{zeroshotgreenstrong}
MPQ
& \textbf{13.65} & \textbf{17.42}
& \textbf{9.65} & \textbf{12.87}
& \textbf{7.81} & \textbf{10.88}
& \textbf{5.30} & \textbf{6.89}
& \textbf{3.78} & \textbf{5.80} \\
\midrule
\multicolumn{11}{@{}l}{\emph{2-bit E8P12 codebook}} \\
LDLQ
& 42.80 & 43.47
& 18.54 & 21.15
& 15.01 & 18.17
& 6.63 & 8.17
& 4.58 & 6.36 \\
LDLQ + CD
& 41.79 & 42.69
& 18.00 & 20.66
& 12.76 & 15.96
& 6.61 & 8.13
& 4.59 & 6.36 \\
QuantEase & 47.31 & 46.24 & 18.84 & 21.17 & 14.78 & 17.77 & 6.67 & 8.19 & 4.70 & 6.46 \\
OJBKQ & 39.28 & 39.72 & 17.55 & 19.94 & 13.84 & 16.57 & 6.55 & 8.11 & 4.58 & 6.36 \\
\rowcolor{zeroshotgreenstrong}
MPQ
& \textbf{31.67} & \textbf{31.48}
& \textbf{15.38} & \textbf{18.10}
& \textbf{11.63} & \textbf{14.77}
& \textbf{6.27} & \textbf{7.76}
& \textbf{4.49} & \textbf{6.30}\\
\bottomrule
\end{tabular}
\end{adjustbox}
\end{threeparttable}
\end{table*}

\subsection{Zero-shot transfer across model families}
We next check whether MPQ's perplexity gains carry over to downstream accuracy. Table~\ref{tab:zeroshot_summary} reports the five task average on Mistral 7B~\citep{Mistral7}, Qwen3 8B~\citep{qwen3}, Llama 3 8B~\citep{llama3}, Llama 2 13B, and Llama 2 70B~\citep{llama2}, covering both the 3-bit scalar and 2-bit E8P12 regimes. MPQ improves over the matched CD polished greedy baseline in every model and bit width, with gains ranging from 0.4 to 1.7 percentage points. Against the stronger QuantEase and OJBKQ baselines, MPQ is best or tied best in nine of ten model/bit width cells. The only exception is Mistral 7B at 3-bit scalar, where QuantEase leads by 0.2 points and MPQ ties OJBKQ for second best. At 2-bit E8P12, MPQ is the best method for all five models, improving over LDLQ+CD by 0.4 to 1.4 points. These results show that MPQ's advantage is not restricted to perplexity and remains visible on downstream multiple choice evaluation, especially in the more constrained 2-bit lattice regime.

\begin{table*}[t]
\centering
\caption{Average zero-shot accuracy in percentage points. Bold denotes the best quantized result within each model and bit width. Underlining denotes second best. Table~\ref{tab:full_zeroshot_avg} gives the detailed accuracy for each task.}
\label{tab:zeroshot_summary}
\begin{adjustbox}{width=0.9\textwidth,center}
\begin{threeparttable}
\footnotesize
\setlength{\tabcolsep}{3.4pt}
\renewcommand{\arraystretch}{1.13}
\begin{tabular}{@{}l c@{\hspace{0.5em}}ccccc@{\hspace{0.8em}}ccccc@{}}
\toprule
\multirow{2}{*}{\textbf{Model}} &
\multirow{2}{*}{\textbf{BF16}} &
\multicolumn{5}{c}{\textbf{3-bit scalar}} &
\multicolumn{5}{c}{\textbf{2-bit E8P12}} \\
\cmidrule(lr){3-7}\cmidrule(lr){8-12}
& &
{\scriptsize \textbf{GPTQ}} &
{\scriptsize \textbf{GPTQ+CD}} &
{\scriptsize \textbf{QuantEase}} &
{\scriptsize \textbf{OJBKQ}} &
{\scriptsize \textbf{MPQ}} &
{\scriptsize \textbf{LDLQ}} &
{\scriptsize \textbf{LDLQ+CD}} &
{\scriptsize \textbf{QuantEase}} &
{\scriptsize \textbf{OJBKQ}} &
{\scriptsize \textbf{MPQ}} \\
\midrule
Mistral\mbox{-}7B
& \tg{74.3}
& 71.2 & 71.2 & \textbf{72.1} & \underline{71.9} & \zgs{\underline{71.9}}
& 66.8 & \underline{67.5} & 66.1 & \underline{67.5} & \zgs{\textbf{67.9}} \\

Qwen3\mbox{-}8B
& \tg{73.5}
& 69.3 & 70.1 & 70.4 & \underline{71.0} & \zgs{\textbf{71.8}}
& \underline{68.1} & 67.2 & 63.5 & 66.5 & \zgs{\textbf{68.4}} \\

Llama3\mbox{-}8B
& \tg{73.0}
& 67.9 & 67.7 & 68.1 & \underline{68.5} & \zgs{\textbf{68.6}}
& 57.4 & \underline{59.3} & 55.3 & 59.2 & \zgs{\textbf{59.7}} \\

Llama2\mbox{-}13B
& \tg{71.6}
& 69.1 & 69.4 & 69.9 & \underline{70.0} & \zgs{\textbf{70.1}}
& 65.5 & 65.5 & 65.8 & \underline{66.0} & \zgs{\textbf{66.9}} \\

Llama2\mbox{-}70B
& \tg{76.7}
& \textbf{75.3} & 74.9 & 75.0 & \underline{75.2} & \zgs{\textbf{75.3}}
& 73.7 & 73.5 & \underline{74.2} & \underline{74.2} & \zgs{\textbf{74.6}} \\
\bottomrule
\end{tabular}
% \begin{tablenotes}[flushleft]
% \scriptsize
% \item We average ARC Challenge, ARC Easy, HellaSwag, PIQA, and Winogrande using \texttt{acc\_norm}.
% \end{tablenotes}
\end{threeparttable}
\end{adjustbox}
\end{table*}

\subsection{Solver scalability}
\label{sec:solver_scalability}

\paragraph{Timing and memory protocol.}
Following the experimental setup in Section~\ref{sec:setup}, we measure all solver-side timings on a single RTX 5090, excluding preprocessing, model export, and evaluation, so that the reported numbers reflect the cost of the solver itself. We use \texttt{torch.compile} and set the number of parallel blocks to 1024 throughout to speed up the time-consuming block CD. Peak memory is reported as the CUDA allocated peak during quantization.

\paragraph{Solver-time scaling across model sizes.}

\begin{wrapfigure}{r}{0.55\textwidth}
\vspace{-1em}
\centering
\includegraphics[width=1\linewidth]{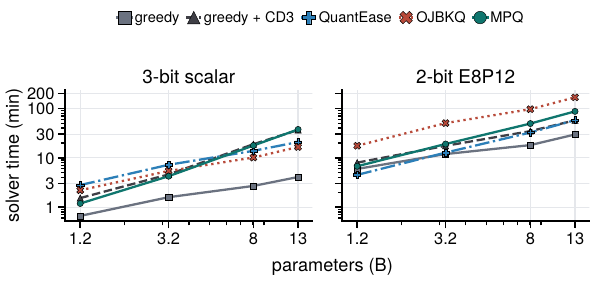}
\caption{Solver-time scaling across Llama model sizes.}
\label{fig:solver_scaling}
\vspace{-0.6em}
\end{wrapfigure}

Figure~\ref{fig:solver_scaling} traces solver-time scaling across four Llama-family scales from 1B to 13B parameters, with detailed numbers in Table~\ref{tab:solver_scaling_timing}. In the 3-bit scalar regime, MPQ matches GPTQ+CD3 across all four scales, but remains several times slower than the greedy GPTQ baseline. In the 2-bit E8P12 regime, this gap narrows substantially. MPQ stays within $3\times$ of LDLQ and is faster than the strongest baseline OJBKQ across all scales. These time and memory costs arise only during offline PTQ.

\paragraph{Speed-quality tradeoff on Llama-3-8B}

\begin{figure}[h]
\centering
\includegraphics[width=0.9\textwidth]{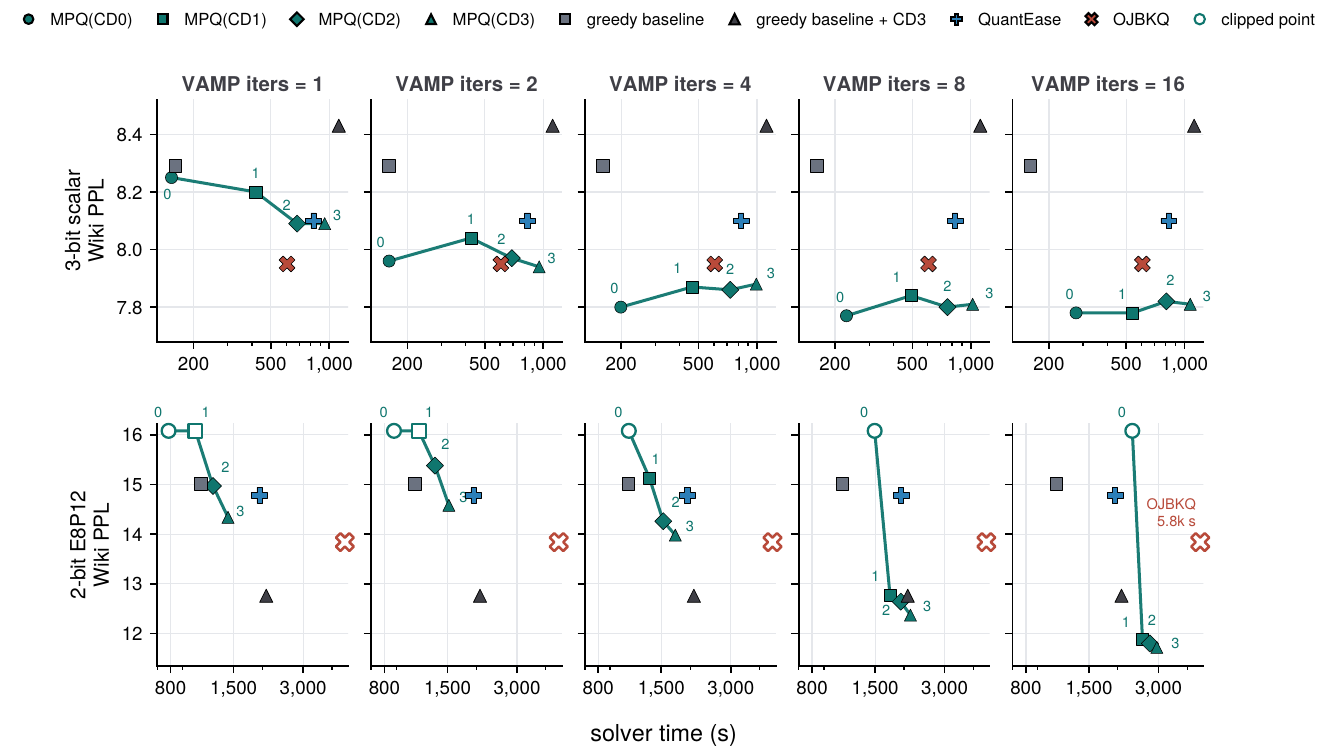}
\caption{Solver-side speed-quality tradeoff on Meta-Llama-3-8B.} %Columns fix VAMP iterations and markers trace CD budgets from none to 3 passes; hollow markers are clipped PPL value.}
\label{fig:speed_quality_main}
\vspace{-10pt}
\end{figure}

Figure~\ref{fig:speed_quality_main} fixes Meta-Llama-3-8B and shows how PPL changes with solver time as VAMP iterations vary across columns and CD passes across markers. Detailed numbers and memory usage are reported in Appendix~\ref{app:runtime}. In the 3-bit scalar regime, 4 VAMP iterations with no CD already reach 7.8 WikiText-2 PPL, below both OJBKQ and QuantEase, and extra CD sweeps barely move the curve. Users in this regime can drop CD entirely and recover most of the wall-clock budget reported in Figure~\ref{fig:solver_scaling}, since VAMP alone already saturates scalar quality. In the 2-bit E8P12 regime, VAMP-only configurations clip above 16 PPL, but each additional CD sweep pushes the operating point lower, reaching around 11.6 PPL at 16 VAMP iterations with three CD passes against LDLQ+CD3 at 12.76. The CD budget is therefore a quality dial that users can spend when targeting aggressive low-bit deployments.

\paragraph{Matched solver budgets.} On Llama-3-8B, MPQ with eight VAMP iterations and three CD sweeps reaches 12.370 WikiText-2 PPL in 2134 seconds, compared with 12.756 in 2076 seconds for LDLQ+CD3. With sixteen VAMP iterations and three CD sweeps, MPQ reaches 11.630 in 2951 seconds, compared with 12.728 in 3057 seconds for LDLQ+CD6 (Table~\ref{tab:matched_solver_budget}).

\paragraph{Deployment.} MPQ quantizes weights offline into GPTQ-compatible scalar or QuIP\# E8P12 formats. With matched grouping and packing, the quantized models use the corresponding inference kernels.

\subsection{Ablation and mechanism study}
\label{sec:mechanism}

Appendices~\ref{app:pilot} and~\ref{app:ablation} cover hyperparameter sensitivity and activation scaling. We next examine codebook choice, CD budget, and the precision and preprocessing diagnostics motivated by Section~\ref{sec:theory}.

\paragraph{Codebook and refinement controls.} With the 2-bit D4 codebook, MPQ has the lowest objective after one CD sweep and the lowest perplexity after three sweeps (Appendix~\ref{app:d4_controls}). With E8P12, MPQ has lower perplexity than LDLQ after 20 CD sweeps (Table~\ref{tab:extended_cd_controls}).

\begin{wrapfigure}{r}{0.55\textwidth}
\vspace{-0.6em}
\centering
\includegraphics[width=1\linewidth]{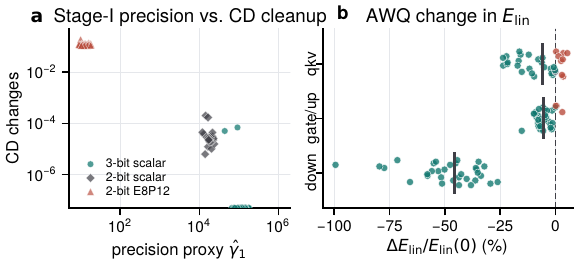}
\caption{Mechanism diagnostics on Meta-Llama-3-8B.}
\label{fig:mechanism_main}
\vspace{-0.8em}
\end{wrapfigure}

\paragraph{VAMP precision and CD cleanup.}
Figure~\ref{fig:mechanism_main}a connects the empirical gains to the warm-start picture in Theorem~\ref{thm:warmstart} by plotting the final VAMP precision proxy \(\hat{\gamma}_1\) against the fraction of assignments changed by CD after VAMP from each module. The 3-bit scalar regime sits at high precision with negligible cleanup, while the 2-bit scalar regime has much higher cleanup ratio. Moreover, the 2-bit E8P12 regime has much lower precision and higher cleanup ratio than 2-bit scalar regime. This pattern indicates that the message-passing stage is near terminal in 3-bit scalar regimes but mainly selects a basin in the harder 2-bit scalar and lattice codebook regime, leaving the precise codeword commitment to CD. Appendix~\ref{app:cleanupdiag} gives the per-model and greedy-baseline cleanup audits.

\paragraph{AWQ and the linear proxy.}
Figure~\ref{fig:mechanism_main}b tests whether AWQ moves the linear proxy in the direction predicted by Proposition~\ref{thm:linear_response} under the same AWQ+RHT scalar setup on Meta-Llama-3-8B. AWQ decreases \(E_{\mathrm{lin}}\) for most scale groups, with the few increases concentrated in \texttt{qkv} and \texttt{gate\_up}, and the sign of the leading-order integrand from Proposition~\ref{thm:linear_response} agrees with the observed change on nearly all groups.  Appendix~\ref{app:cleanupdiag} gives the full sign-count breakdown.

\section{Conclusion}
\label{sec:conclusion}
We presented MPQ, a solver-side approach to post-training quantization that casts row-wise  objective as posterior inference over discrete codebooks. This unified view explains classical greedy solvers as hard-decision limits, while exposing finite-temperature message-passing as a natural way to propagate codeword uncertainty across blocks before committing to a quantized solution. 
The benefit of postponing commitment is most visible in aggressive low-bit regimes, which are also the regimes that determine whether large language models can be deployed on commodity hardware. MPQ improves over strong solver baselines across these regimes, with the clearest gains at 2 bits. We believe message-passing inference is particularly well suited to this combinatorial problem, and will open a productive regime in the design space of low-bit quantization solvers.

% Building on this perspective, MPQ combines VAMP with exact local refinement, yielding stronger perplexity and zero-shot accuracy across scalar and lattice quantization regimes. Our analysis connects preprocessing and MAP projection accuracy through the effective VAMP precision, providing a qualitative interpretation for the message-passing mechanism. We believe that message-passing inference over the discrete codebook posterior will reshape the design space of low-bit quantization.

% Explicit boundary for the ten-page main text.
\label{cr:last-main-page}
\clearpage
\bibliographystyle{unsrt}
\bibliography{ref}

\newpage
\appendix

\onecolumn
\begin{center}
    {\LARGE \bf Appendix}
\end{center}

\section{Additional experimental details}
\label{appendix:experiments}
\providecommand{\wiki}{WikiText-2}
\providecommand{\cFour}{C4}
\providecommand{\vamp}{VAMP}
\providecommand{\cd}{CD}
\providecommand{\stagei}{Stage I}
\providecommand{\stageii}{Stage II}

\subsection{Pilot sweep for frozen \vamp{} hyperparameters}
\label{app:pilot}

Before freezing the defaults, we ran a small pilot sweep over the three scalar VAMP hyperparameters that most directly affect convergence: the number of VAMP iterations, the damping coefficient, and the anchoring multiplier $\tau$. Figure~\ref{fig:stability_appendix} summarizes the pilot together with the calibration-budget robustness check; hollow markers denote the selected defaults. The chosen setting \((T_{\mathrm{vamp}},\mathrm{damping},\tau)=(16,0.5,0.2)\) lies in the stable region of the sweep rather than at a brittle edge, and the Llama-3-8B 3-bit result remains stable as the calibration budget varies from 32 to 128 sequences with 3 different random seeds.

\begin{figure}[h]
\centering
\makebox[\textwidth][c]{\includegraphics[width=1\textwidth]{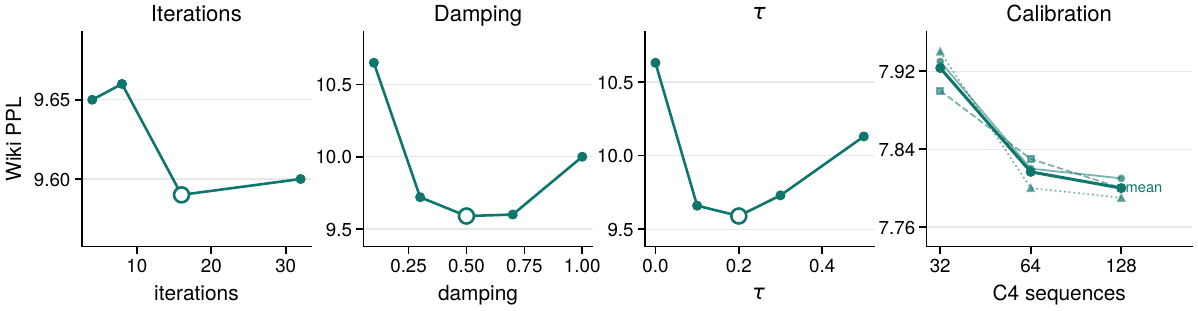}}
\caption{Pilot sweeps over VAMP iterations, damping, anchoring multiplier $\tau$, and calibration budget for Llama-3-8B 3-bit scalar quantization.}
\label{fig:stability_appendix}
\end{figure}

\begin{table}[!h]
\centering
\tiny
\setlength{\tabcolsep}{1.8pt}
\renewcommand{\arraystretch}{1.06}
\caption{Full perplexity matrix for the retained seven-model evaluation. Entries are WikiText-2 / C4 PPL.}
\label{tab:full_ppl_matrix}
\begin{adjustbox}{width=\textwidth}
\begin{tabular}{lccccccc}
\toprule
\textbf{Method}
& \textbf{Llama3.2-1B}
& \textbf{Llama3.2-3B}
& \textbf{Llama3-8B}
& \textbf{Llama2-13B}
& \textbf{Llama2-70B}
& \textbf{Qwen3-8B}
& \textbf{Mistral-7B} \\
\midrule
\rowcolor{tablegray}
BF16 & 9.75 / 12.77 & 7.82 / 10.47 & 6.14 / 8.93 & 4.89 / 6.47 & 3.32 / 5.52 & 6.97 / 10.19 & 5.32 / 7.84 \\
\midrule
\multicolumn{8}{@{}l}{\emph{4-bit scalar codebook, group size 128}} \\
GPTQ & 10.63 / 13.86 & 8.22 / 11.08 & 6.54 / 9.41 & 4.99 / 6.57 & 3.43 / 5.58 & 7.22 / 10.58 & 5.42 / 7.97 \\
GPTQ + CD & 10.57 / 13.83 & 8.24 / 11.10 & 6.55 / 9.42 & 4.99 / 6.57 & 3.43 / 5.58 & 7.26 / 10.59 & 5.43 / 7.96 \\
QuantEase & 10.47 / 13.80 & 8.18 / 11.03 & 6.50 / 9.37 & 4.98 / 6.57 & 3.43 / 5.58 & 7.19 / 10.54 & \textbf{5.41} / 7.95 \\
OJBKQ & 10.43 / 13.68 & 8.15 / 10.99 & 6.48 / 9.35 & 4.97 / 6.56 & 3.42 / 5.58 & \textbf{7.16} / 10.53 & \textbf{5.41} / \textbf{7.94} \\
MPQ & \textbf{10.38 / 13.63} & \textbf{8.12 / 10.94} & \textbf{6.46 / 9.33} & \textbf{4.96 / 6.55} & \textbf{3.41 / 5.57} & \textbf{7.16 / 10.51} & \textbf{5.41 / 7.94} \\
\midrule
\multicolumn{8}{@{}l}{\emph{3-bit scalar codebook, group size 128}} \\
GPTQ & 15.45 / 19.25 & 10.56 / 13.76 & 8.27 / 11.59 & 5.46 / 7.09 & 3.93 / 5.89 & 8.19 / 11.53 & 5.94 / 8.56 \\
GPTQ + CD & 15.58 / 18.94 & 10.51 / 13.82 & 8.37 / 11.60 & 5.46 / 7.09 & 3.92 / 5.90 & 8.21 / 11.56 & 5.94 / 8.54 \\
QuantEase & 14.86 / 18.51 & 10.05 / 13.36 & 8.10 / 11.24 & 5.39 / 6.95 & 3.88 / 5.88 & 8.03 / 11.29 & 5.82 / 8.39 \\
OJBKQ & 14.26 / 18.00 & 9.93 / 13.22 & 7.95 / 11.08 & 5.40 / 6.98 & 3.87 / 5.86 & 7.89 / 11.24 & 5.81 / 8.36 \\
MPQ & \textbf{13.65 / 17.42} & \textbf{9.65 / 12.87} & \textbf{7.81 / 10.88} & \textbf{5.30 / 6.89} & \textbf{3.78 / 5.80} & \textbf{7.80 / 11.13} & \textbf{5.80 / 8.35} \\
\midrule
\multicolumn{8}{@{}l}{\emph{2-bit E8P12 codebook}} \\
LDLQ & 42.80 / 43.47 & 18.54 / 21.15 & 15.01 / 18.17 & 6.63 / 8.17 & 4.58 / 6.36 & 9.54 / 12.78 & 7.81 / 10.53 \\
LDLQ + CD & 41.79 / 42.69 & 18.00 / 20.66 & 12.76 / 15.96 & 6.61 / 8.13 & 4.59 / 6.36 & 9.41 / 12.76 & 6.98 / 9.53 \\
QuantEase & 47.31 / 46.24 & 18.84 / 21.17 & 14.78 / 17.77 & 6.67 / 8.19 & 4.70 / 6.46 & 10.55 / 13.80 & 7.48 / 10.10 \\
OJBKQ & 39.28 / 39.72 & 17.55 / 19.94 & 13.84 / 16.57 & 6.55 / 8.11 & 4.58 / 6.36 & 9.43 / 12.81 & 6.99 / 9.53 \\
MPQ & \textbf{31.67 / 31.48} & \textbf{15.38 / 18.10} & \textbf{11.63 / 14.77} & \textbf{6.27 / 7.76} & \textbf{4.49 / 6.30} & \textbf{9.19 / 12.57} & \textbf{6.75 / 9.38} \\

\bottomrule
\end{tabular}
\end{adjustbox}
\vspace{0.8em}
\end{table}
\subsection{Full perplexity matrix}
\label{app:full_ppl}

Table~\ref{tab:full_ppl_matrix} expands the main perplexity comparison to the complete seven-model set used in the paper: five Llama-family models, Qwen3-8B, and Mistral-7B. Each entry is WikiText-2 / C4 perplexity, and MPQ denotes the final CD-polished configuration. The table section headers specify the codebook and preprocessing regime.

\subsection{Full zero-shot task matrix}
\label{app:full_zeroshot}

Table~\ref{tab:full_zeroshot_avg} reports the retained zero-shot task matrix. Bold marks the best quantized result within each model, bit-width, and task; underlining marks the second-best.

\begin{table}[H]
\centering
\scriptsize
\setlength{\tabcolsep}{4.0pt}
\renewcommand{\arraystretch}{1.06}
\caption{Full zero-shot task matrix for the retained models.}
\label{tab:full_zeroshot_avg}
\begin{adjustbox}{width=0.9\textwidth}
\begin{tabular}{llcccccc}
\toprule
\textbf{Model / Bits} & \textbf{Method} & \textbf{ARC-C} & \textbf{ARC-E} & \textbf{Hella} & \textbf{PIQA} & \textbf{Wino} & \textbf{Avg} \\
\midrule
\multicolumn{8}{@{}l}{\emph{Llama-3-8B}} \\
BF16 reference & \tg{BF16} & \tg{54.1} & \tg{77.6} & \tg{79.2} & \tg{80.6} & \tg{73.6} & \tg{73.0} \\
3-bit scalar & GPTQ & 44.5 & \underline{72.2} & 74.3 & \textbf{77.4} & \textbf{70.9} & 67.9 \\
3-bit scalar & GPTQ + CD & 45.1 & 72.0 & \underline{74.5} & \textbf{77.4} & 69.6 & 67.7 \\
3-bit scalar & QuantEase & 46.5 & 70.5 & 75.0 & \underline{77.3} & 71.2 & 68.1 \\
3-bit scalar & OJBKQ & \textbf{47.9} & 72.4 & 75.2 & 76.8 & 70.3 & \underline{68.5} \\
3-bit scalar & MPQ & \underline{47.6} & \textbf{73.4} & \textbf{75.5} & 76.1 & \underline{70.5} & \textbf{68.6} \\
2-bit E8P12 & LDLQ & 34.6 & 57.4 & 59.8 & 70.9 & 64.4 & 57.4 \\
2-bit E8P12 & LDLQ + CD & \underline{35.6} & \textbf{60.9} & 63.0 & \underline{72.1} & 64.7 & \underline{59.3} \\
2-bit E8P12 & QuantEase & 30.8 & 53.9 & 59.1 & 69.9 & 62.9 & 55.3 \\
2-bit E8P12 & OJBKQ & \underline{35.6} & \underline{59.2} & \underline{64.4} & 71.4 & \underline{65.2} & 59.2 \\
2-bit E8P12 & MPQ & \textbf{35.8} & 58.6 & \textbf{64.5} & \textbf{72.9} & \textbf{66.9} & \textbf{59.7} \\
\midrule
\multicolumn{8}{@{}l}{\emph{Qwen3-8B}} \\
BF16 reference & \tg{BF16} & \tg{56.9} & \tg{79.8} & \tg{78.7} & \tg{79.5} & \tg{72.8} & \tg{73.5} \\
3-bit scalar & GPTQ & 49.1 & 75.2 & 74.7 & \underline{78.5} & 69.0 & 69.3 \\
3-bit scalar & GPTQ + CD & 51.7 & 76.4 & 75.0 & 78.2 & 69.4 & 70.1 \\
3-bit scalar & QuantEase & \underline{52.8} & 77.4 & 74.9 & 77.8 & 69.2 & 70.4 \\
3-bit scalar & OJBKQ & 51.4 & \underline{78.6} & \underline{75.3} & \textbf{78.8} & \textbf{71.2} & \underline{71.0} \\
3-bit scalar & MPQ & \textbf{54.4} & \textbf{79.6} & \textbf{75.9} & 78.1 & \underline{70.8} & \textbf{71.8} \\
2-bit E8P12 & LDLQ & \textbf{49.0} & \underline{76.2} & 69.4 & \underline{76.7} & 69.3 & \underline{68.1} \\
2-bit E8P12 & LDLQ + CD & 47.3 & 72.8 & \underline{70.0} & 76.5 & \underline{69.5} & 67.2 \\
2-bit E8P12 & QuantEase & 40.5 & 66.0 & 67.8 & 75.9 & 67.5 & 63.5 \\
2-bit E8P12 & OJBKQ & 45.0 & 72.1 & \textbf{70.3} & 76.2 & 69.1 & 66.5 \\
2-bit E8P12 & MPQ & \underline{48.7} & \textbf{76.6} & \underline{70.0} & \textbf{77.0} & \textbf{69.6} & \textbf{68.4} \\
\midrule
\multicolumn{8}{@{}l}{\emph{Mistral-7B}} \\
BF16 reference & \tg{BF16} & \tg{54.3} & \tg{80.1} & \tg{80.7} & \tg{82.0} & \tg{74.3} & \tg{74.3} \\
3-bit scalar & GPTQ & 49.6 & 75.4 & 78.0 & 81.0 & \underline{72.1} & 71.2 \\
3-bit scalar & GPTQ + CD & 50.0 & 75.3 & 77.3 & \textbf{81.6} & 71.9 & 71.2\\
3-bit scalar & QuantEase & \textbf{51.7} & \textbf{77.9} & 78.3 & 80.7 & 72.0 & \textbf{72.1} \\
3-bit scalar & OJBKQ & \underline{50.7} & \underline{76.8} & \textbf{79.0} & \underline{81.1} & 72.0 & \underline{71.9} \\
3-bit scalar & MPQ & 50.5 & 76.2 & \underline{78.6} & 81.0 & \textbf{73.2} & \underline{71.9} \\
2-bit E8P12 & LDLQ & 41.5 & 70.1 & \textbf{73.7} & 77.6 & \underline{71.3} & 66.8 \\
2-bit E8P12 & LDLQ + CD & 43.1 & 70.8 & \underline{73.4} & \underline{78.8} & \textbf{71.4} & \underline{67.5} \\
2-bit E8P12 & QuantEase & 42.6 & 70.0 & 70.6 & 78.1 & 69.0 & 66.1 \\
2-bit E8P12 & OJBKQ & \underline{43.5} & \textbf{72.2} & \underline{73.4} & 78.0 & 70.6 & \underline{67.5} \\
2-bit E8P12 & MPQ & \textbf{43.8} & \underline{71.3} & \underline{73.4} & \textbf{79.8} & 71.1 & \textbf{67.9} \\

\bottomrule
\end{tabular}
\end{adjustbox}
\end{table}

\begin{table}[H]
\centering
\scriptsize
\setlength{\tabcolsep}{4.0pt}
\renewcommand{\arraystretch}{0.98}
\vspace{-1.6em}
\addtocounter{table}{-1}
\renewcommand{\thetable}{\arabic{table} continued}
\caption{Full zero-shot task matrix for the retained models.}
\renewcommand{\thetable}{\arabic{table}}
\begin{adjustbox}{width=0.9\textwidth}
\begin{tabular}{llcccccc}
\toprule
\textbf{Model / Bits} & \textbf{Method} & \textbf{ARC-C} & \textbf{ARC-E} & \textbf{Hella} & \textbf{PIQA} & \textbf{Wino} & \textbf{Avg} \\
\midrule
% \multicolumn{8}{@{}l}{\emph{Mistral-7B}} \\
% BF16 reference & \tg{BF16} & \tg{54.3} & \tg{80.1} & \tg{80.7} & \tg{82.0} & \tg{74.3} & \tg{74.3} \\
% 3-bit scalar & GPTQ & 49.6 & 75.4 & 78.0 & 81.0 & \underline{72.1} & 71.2 \\
% 3-bit scalar & GPTQ + CD & 50.0 & 75.3 & 77.3 & \textbf{81.6} & 71.9 & 71.2\\
% 3-bit scalar & QuantEase & \textbf{51.7} & \textbf{77.9} & 78.3 & 80.7 & 72.0 & \textbf{72.1} \\
% 3-bit scalar & OJBKQ & \underline{50.7} & \underline{76.8} & \textbf{79.0} & \underline{81.1} & 72.0 & \underline{71.9} \\
% 3-bit scalar & MPQ & 50.5 & 76.2 & \underline{78.6} & 81.0 & \textbf{73.2} & \underline{71.9} \\
% 2-bit E8P12 & LDLQ & 41.5 & 70.1 & \textbf{73.7} & 77.6 & \underline{71.3} & 66.8 \\
% 2-bit E8P12 & LDLQ + CD & 43.1 & 70.8 & \underline{73.4} & \underline{78.8} & \textbf{71.4} & \underline{67.5} \\
% 2-bit E8P12 & QuantEase & 42.6 & 70.0 & 70.6 & 78.1 & 69.0 & 66.1 \\
% 2-bit E8P12 & OJBKQ & \underline{43.5} & \textbf{72.2} & \underline{73.4} & 78.0 & 70.6 & \underline{67.5} \\
% 2-bit E8P12 & MPQ & \textbf{43.8} & \underline{71.3} & \underline{73.4} & \textbf{79.8} & 71.1 & \textbf{67.9} \\
% \midrule
\multicolumn{8}{@{}l}{\emph{Llama-2-13B}} \\
BF16 reference & \tg{BF16} & \tg{48.8} & \tg{76.5} & \tg{79.6} & \tg{80.4} & \tg{72.5} & \tg{71.6} \\
3-bit scalar & GPTQ & 45.5 & 73.0 & 76.5 & 79.2 & \underline{71.4} & 69.1 \\
3-bit scalar & GPTQ + CD & 46.1 & \underline{73.7} & 76.7 & 79.1 & 71.2 & 69.4 \\
3-bit scalar & QuantEase & \underline{47.5} & 73.3 & \underline{77.5} & 79.5 & \textbf{71.8} & 69.9 \\
3-bit scalar & OJBKQ & 47.2 & 74.5 & 77.1 & \textbf{80.5} & 70.7 & \underline{70.0} \\
3-bit scalar & MPQ & \textbf{47.8} & 73.5 & \textbf{77.9} & \underline{80.1} & 71.0 & \textbf{70.1} \\
2-bit E8P12 & LDLQ & 41.1 & 68.9 & 72.1 & 76.8 & 68.6 & 65.5 \\
2-bit E8P12 & LDLQ + CD & 40.4 & 68.1 & 72.0 & \underline{77.1} & \underline{69.9} & 65.5 \\
2-bit E8P12 & QuantEase & \underline{41.9} & 69.1 & 71.1 & \textbf{77.9} & 69.1 & 65.8 \\
2-bit E8P12 & OJBKQ & 41.6 & \underline{70.0} & \underline{72.3} & 77.0 & 69.2 & \underline{66.0} \\
2-bit E8P12 & MPQ & \textbf{43.0} & \textbf{71.0} & \textbf{73.0} & \underline{77.1} & \textbf{70.3} & \textbf{66.9} \\
\midrule
\multicolumn{8}{@{}l}{\emph{Llama-2-70B}} \\
BF16 reference & \tg{BF16} & \tg{53.9} & \tg{82.6} & \tg{84.4} & \tg{82.5} & \tg{80.0} & \tg{76.7} \\
3-bit scalar & GPTQ & \underline{53.7} & \textbf{79.1} & 82.5 & \textbf{82.3} & \underline{79.1} & \textbf{75.3} \\
3-bit scalar & GPTQ + CD & 53.5 & 77.1 & 82.5 & \underline{82.2} & \textbf{79.2} & 74.9 \\
3-bit scalar & QuantEase & 53.6 & 78.5 & \underline{83.1} & 81.6 & 78.4 & 75.0 \\
3-bit scalar & OJBKQ & \underline{53.7} & \underline{78.6} & 83.0 & \textbf{82.3} & 78.6 & \underline{75.2} \\
3-bit scalar & MPQ & \textbf{53.9} & 78.2 & \textbf{83.6} & 82.1 & 78.6 & \textbf{75.3} \\
2-bit E8P12 & LDLQ & 51.8 & 77.2 & \underline{80.5} & 81.0 & 77.9 & 73.7 \\
2-bit E8P12 & LDLQ + CD & 51.0 & 77.1 & \underline{80.5} & 81.1 & 77.6 & 73.5 \\
2-bit E8P12 & QuantEase & 53.2 & 78.7 & 80.2 & 80.8 & \textbf{78.3} & \underline{74.2} \\
2-bit E8P12 & OJBKQ & \underline{53.4} & \underline{78.2} & 80.4 & \underline{81.4} & 77.4 & \underline{74.2} \\
2-bit E8P12 & MPQ & \textbf{53.8} & \textbf{78.7} & \textbf{80.8} & \textbf{81.6} & \underline{78.0} & \textbf{74.6} \\
\bottomrule
\end{tabular}
\end{adjustbox}
\end{table}

\subsection{Ablation study}
\label{app:ablation}

Figure~\ref{fig:ablation_appendix} shows the MPQ recipe variants on the middle size 8B/7B models, while Table~\ref{tab:preprocess_ablation} tests how AWQ affects MPQ solvers under the same RHT convention.

\vspace{-0.75em}
\begin{figure}[H]
\centering
\includegraphics[width=\textwidth]{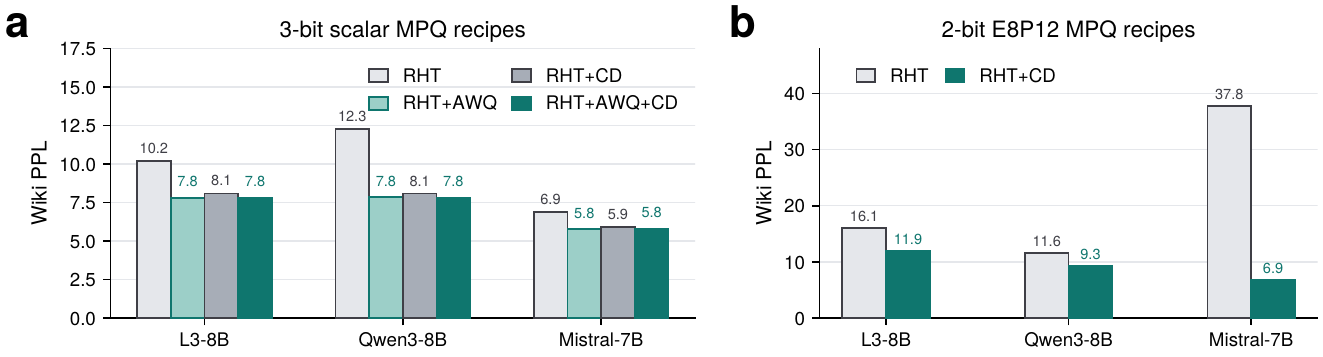}
\caption{MPQ recipe variants on WikiText-2 for the 8B/7B models. \textbf{a} Performance comparison of 3-bit scalar preprocessing and CD choices. \textbf{b} Performance comparison of 2-bit E8P12 VAMP and CD-polished recipes. AWQ preprocessing intentionally manipulates the per-channel activation magnitude which would disrupt the lattice codebook’s block-level isotropy and make quantization fail. Therefore, AWQ is not inlcuded in the 2-bit E8P12 setting for preprocessing ablation study.} 
\label{fig:ablation_appendix}
\end{figure}

\vspace{-0.75em}
\begin{table}[h]
\centering
\scriptsize
\setlength{\tabcolsep}{3.5pt}
\caption{3-bit scalar preprocessing and CD ablation on WikiText-2 / C4 perplexity. Delta denotes AWQ+RHT minus RHT-only.}
\label{tab:preprocess_ablation}
\begin{adjustbox}{width=0.9\textwidth}
\begin{tabular}{llccc}
\toprule
\textbf{Model} & \textbf{Method} & \textbf{RHT only} & \textbf{AWQ + RHT} & \textbf{Delta} \\
\midrule
\multirow{2}{*}{Llama3.2-1B}
% & GPTQ & 15.23 / 18.96 & 15.45 / 19.25 & +0.22 / +0.29 \\
% & GPTQ + CD & 14.95 / 18.61 & 15.58 / 18.94 & +0.63 / +0.33 \\
& MPQ (VAMP-only) & 31.24 / 36.36 & 13.76 / 17.71 & -17.48 / -18.65 \\
& MPQ & 14.55 / 18.33 & 13.65 / 17.42 & -0.90 / -0.91 \\
\midrule
\multirow{2}{*}{Llama3.2-3B}
% & GPTQ & 10.43 / 13.73 & 10.56 / 13.76 & +0.13 / +0.03 \\
% & GPTQ + CD & 10.47 / 13.77 & 10.51 / 13.82 & +0.04 / +0.05 \\
& MPQ (VAMP-only) & 14.76 / 18.81 & 9.59 / 12.92 & -5.17 / -5.89 \\
& MPQ & 10.07 / 13.33 & 9.65 / 12.87 & -0.42 / -0.46 \\
\midrule
\multirow{2}{*}{Llama3-8B}
% & GPTQ & 8.22 / 11.51 & 8.29 / 11.61 & +0.07 / +0.10 \\
% & GPTQ + CD & 8.31 / 11.52 & 8.43 / 11.63 & +0.12 / +0.11 \\
& MPQ (VAMP-only) & 10.18 / 13.72 & 7.80 / 10.96 & -2.38 / -2.76 \\
& MPQ & 8.09 / 11.24 & 7.81 / 10.88 & -0.28 / -0.36 \\
\midrule
\multirow{2}{*}{Qwen3-8B}
% & GPTQ & 8.11 / 11.42 & 8.39 / 11.88 & +0.28 / +0.46 \\
% & GPTQ + CD & 8.10 / 11.40 & 8.40 / 11.86 & +0.30 / +0.46 \\
& MPQ (VAMP-only) & 12.27 / 17.77 & 7.85 / 11.18 & -4.42 / -6.59 \\
& MPQ & 8.08 / 11.37 & 7.80 / 11.13 & -0.28 / -0.24 \\
\midrule
\multirow{2}{*}{Mistral-7B}
% & GPTQ & 5.94 / 8.51 & 5.94 / 8.56 & +0.00 / +0.05 \\
% & GPTQ + CD & 5.94 / 8.50 & 5.94 / 8.54 & +0.00 / +0.04 \\
& MPQ (VAMP-only) & 6.87 / 9.96 & 5.78 / 8.37 & -1.09 / -1.59 \\
& MPQ & 5.89 / 8.49 & 5.80 / 8.35 & -0.09 / -0.15 \\
\bottomrule
\end{tabular}
\end{adjustbox}
\end{table}

\subsection{Speed-quality tradeoff and memory}
\label{app:runtime}

We measure solver-side cost on a single RTX 5090 and exclude Hessian construction, evaluation, and export from the timings. Figure~\ref{fig:solver_scaling} plots scaling against exact model parameter counts, and Table~\ref{tab:solver_scaling_timing} reports the same solver times in minutes.

Tables~\ref{tab:runtime_memory_3bit} and~\ref{tab:runtime_memory_2bit} give the full Meta-Llama-3-8B grids corresponding to Figure~\ref{fig:speed_quality_main}.
\begin{table}[h]
\centering
\small
\setlength{\tabcolsep}{5pt}
\caption{Solver-time scaling on a single RTX 5090.}
\label{tab:solver_scaling_timing}
\begin{adjustbox}{width=0.9\textwidth}
\begin{tabular}{llcccc}
\toprule
\textbf{Bits} & \textbf{Method} & \textbf{Llama3.2-1B} & \textbf{Llama3.2-3B} & \textbf{Llama3-8B} & \textbf{Llama2-13B} \\
\midrule
\multirow{3}{*}{3-bit} & GPTQ & 0.7 & 1.6 & 2.7 & 4.1 \\
& GPTQ+CD3 & 1.5 & 4.6 & 18.8 & 36.5 \\
& QuantEase & 2.8 & 7.2 & 13.8 & 20.8 \\
& OJBKQ & 2.2 & 5.4 & 10.1 & 16.5 \\
& MPQ & 1.2 & 4.2 & 17.7 & 37.4 \\
\midrule
\multirow{3}{*}{2-bit} & LDLQ & 5.9 & 11.9 & 18.0 & 29.6 \\
& LDLQ+CD3 & 8.0 & 17.4 & 34.6 & 57.2 \\
& QuantEase & 4.5 & 12.7 & 32.4 & 57.6 \\
& OJBKQ & 17.5 & 50.3 & 96.1 & 166.9 \\
& MPQ & 6.7 & 19.0 & 49.2 & 86.9 \\
\bottomrule
\end{tabular}
\end{adjustbox}
\end{table}

\begin{table}[h]
\centering
\scriptsize
\setlength{\tabcolsep}{3pt}
\caption{3-bit scalar codebook speed-quality grid on Meta-Llama-3-8B.}
\label{tab:runtime_memory_3bit}
\begin{threeparttable}
\begin{adjustbox}{width=0.9\textwidth}
\begin{tabular}{lcccc}
\toprule
\textbf{Method} & \textbf{\wiki{}} & \textbf{\cFour{}} & \textbf{Solver time (s)} & \textbf{Memory peak (GB)} \\
\midrule
GPTQ & 8.29 & 11.61 & 161 & 8.21 \\
GPTQ + CD3 & 8.43 & 11.63 & 1128 & 13.20 \\
QuantEase & 8.10 & 11.24 & 829 & 17.93 \\
OJBKQ & 7.95 & 11.03 & 606 & 15.84 \\
\midrule
MPQ-1 (VAMP-only) & 8.25 & 11.61 & 154 & 12.61 \\
MPQ-2 (VAMP-only) & 7.96 & 11.20 & 161 & 13.99 \\
MPQ-4 (VAMP-only) & 7.80 & 11.00 & 199 & 13.37 \\
MPQ-8 (VAMP-only) & 7.77 & 10.95 & 229 & 13.29 \\
MPQ-16 (VAMP-only) & 7.78 & 10.94 & 276 & 13.06 \\
\midrule
MPQ-1 + CD1 & 8.20 & 11.39 & 418 & 13.63 \\
MPQ-2 + CD1 & 8.04 & 11.14 & 425 & 13.52 \\
MPQ-4 + CD1 & 7.87 & 11.02 & 463 & 13.65 \\
MPQ-8 + CD1 & 7.84 & 10.94 & 493 & 13.44 \\
MPQ-16 + CD1 & 7.78 & 10.91 & 540 & 13.57 \\
\midrule
MPQ-1 + CD2 & 8.09 & 11.25 & 683 & 13.09 \\
MPQ-2 + CD2 & 7.97 & 11.11 & 690 & 13.30 \\
MPQ-4 + CD2 & 7.86 & 11.00 & 728 & 13.22 \\
MPQ-8 + CD2 & 7.80 & 10.95 & 758 & 14.30 \\
MPQ-16 + CD2 & 7.82 & 10.94 & 805 & 13.67 \\
\midrule
MPQ-1 + CD3 & 8.09 & 11.24 & 947 & 13.40 \\
MPQ-2 + CD3 & 7.94 & 11.08 & 954 & 13.50 \\
MPQ-4 + CD3 & 7.88 & 11.00 & 992 & 13.52 \\
MPQ-8 + CD3 & 7.81 & 10.92 & 1022 & 13.00 \\
MPQ-16 + CD3 & 7.81 & 10.87 & 1069 & 13.55 \\
\bottomrule
\end{tabular}
\end{adjustbox}
\end{threeparttable}
\end{table}

\begin{table}[h]
\centering
\scriptsize
\setlength{\tabcolsep}{3pt}
\caption{2-bit E8P12 speed-quality grid on Meta-Llama-3-8B.}
\label{tab:runtime_memory_2bit}
\begin{adjustbox}{width=0.9\textwidth}
\begin{tabular}{lcccc}
\toprule
\textbf{Method} & \textbf{\wiki{}} & \textbf{\cFour{}} & \textbf{Solver time (s)} & \textbf{Memory peak (GB)} \\
\midrule
LDLQ & 15.01 & 18.17 & 1082 & 8.45 \\
LDLQ + CD3 & 12.76 & 15.96 & 2076 & 17.49 \\
QuantEase & 14.78 & 17.77 & 1944 & 18.66 \\
OJBKQ & 13.84 & 16.57 & 5768 & 21.49 \\
\midrule
MPQ-1 (VAMP-only) & 321.12 & 280.36 & 784 & 22.63 \\
MPQ-1 + CD1 & 16.69 & 19.63 & 1024 & 22.63 \\
MPQ-1 + CD2 & 14.97 & 18.27 & 1221 & 22.63 \\
MPQ-1 + CD3 & 14.34 & 17.64 & 1418 & 22.63 \\
MPQ-2 (VAMP-only) & 742.85 & 502.59 & 880 & 22.63 \\
MPQ-2 + CD1 & 17.66 & 21.01 & 1127 & 22.86 \\
MPQ-2 + CD2 & 15.38 & 18.50 & 1324 & 22.86 \\
MPQ-2 + CD3 & 14.58 & 17.87 & 1521 & 22.86 \\
MPQ-4 (VAMP-only) & 590.20 & 446.48 & 1085 & 22.86 \\
MPQ-4 + CD1 & 15.13 & 18.22 & 1333 & 22.86 \\
MPQ-4 + CD2 & 14.26 & 17.26 & 1530 & 22.86 \\
MPQ-4 + CD3 & 13.98 & 16.93 & 1725 & 22.86 \\
MPQ-8 (VAMP-only) & 241.49 & 153.71 & 1493 & 22.86 \\
MPQ-8 + CD1 & 12.77 & 16.18 & 1740 & 22.86 \\
MPQ-8 + CD2 & 12.64 & 16.07 & 1937 & 22.86 \\
MPQ-8 + CD3 & 12.37 & 15.77 & 2134 & 22.86 \\
MPQ-16 (VAMP-only) & 17.24 & 21.01 & 2311 & 22.86 \\
MPQ-16 + CD1 & 11.89 & 14.99 & 2558 & 22.86 \\
MPQ-16 + CD2 & 11.80 & 14.81 & 2754 & 22.86 \\
MPQ-16 + CD3 & 11.63 & 14.77 & 2951 & 22.86 \\
\bottomrule
\end{tabular}
\end{adjustbox}
\end{table}

\subsection{Theory-guided cleanup diagnostic}
\label{app:cleanupdiag}
\paragraph{VAMP precision and cleanup.}
The VAMP cleanup diagnostic summarized in Figure~\ref{fig:mechanism_main} was run on the full Meta-Llama-3-8B model. In 3-bit scalar quantization, the median module-level final precision proxy \(\hat{\gamma}_1\) is \(1.77\times10^5\) over 224 modules, the median cleanup fraction is zero, and the global cleanup fraction is \(6.44\times10^{-7}\). In 2-bit E8P12, the median \(\hat{\gamma}_1\) drops to \(1.30\times10^1\) over 128 fused-module checkpoints, the median cleanup fraction rises to \(1.27\times10^{-1}\), and the global cleanup fraction is \(1.29\times10^{-1}\). This result supports the mechanism that when VAMP is sharp, CD is nearly a no-op; when the codebook is harder, CD performs substantial local repair.

\FloatBarrier

\paragraph{AWQ \(E_{\mathrm{lin}}\) diagnostic.}
We also checked the sign of the linearized error proxy from Eq.~\eqref{eq:Elin_derivative_two} (equivalently, the integrand of Corollary~\ref{cor:coordinate_response_elin}, Eq.~\eqref{eq:coordinate_response_elin}) on Meta-Llama-3-8B. To match the un-rotated path of Proposition~\ref{thm:linear_response} directly, we work in the RHT basis throughout: we apply RHT to the model, collect the calibration Hessian $H_0$ in the rotated basis from 128 C4 sequences of length 2048, and run GPTQModel AWQ scale search on the same rotated model with group size 128 to obtain $D=\operatorname{diag}(1/s)$. For each AWQ scale group we compute $E_0=\tfrac{1}{n}\operatorname{tr}[(\beta_0 H_0+\gamma I)^{-1}]$, $E_1=\tfrac{1}{n}\operatorname{tr}[(\beta_0 D H_0 D+\gamma I)^{-1}]$ and the difference $\Delta E_{\mathrm{lin}}=E_1-E_0$ at $\beta_0=\gamma=1$, together with the $t=0$ integrand. By Lemma~\ref{lem:orthogonal_invariance_elin}, this matches the production pipeline (AWQ before RHT) up to orthogonal similarity. AWQ decreased \(E_{\mathrm{lin}}\) on 84/96 unique AWQ scale groups, increased it on 11/96 groups, and was numerically zero on one of the 96 groups. Figure~\ref{fig:mechanism_main}b shows the distribution of relative changes. The module-group breakdown is simple, with all \texttt{down\_proj} groups decreasing and the few increases concentrated in \texttt{qkv} and \texttt{gate\_up}. The sign of the \(t=0\) integrand matched the observed \(\Delta E_{\mathrm{lin}}\), after accounting for the theorem's leading minus sign, on 95/96 groups, with the only mismatch having \(|\Delta E_{\mathrm{lin}}|\approx 7.5\times10^{-11}\). The largest decreases occur in early \texttt{down\_proj} groups, while all positive absolute changes are below \(8\times10^{-7}\). This supports the qualitative remark that AWQ usually lowers the linear-module proxy in the AWQ+RHT setting.

This diagnostic makes the effect of CD more clearer. In the 3-bit scalar GPTQ setting, CD changes only a small fraction of final scalar assignments and gives a modest surrogate drop, matching the small and noisy PPL movement. In the 2-bit E8P12 LDLQ setting, CD performs real local repair and improves PPL substantially, but it still cannot match MPQ, suggesting that the remaining gap is a VAMP basin-selection issue rather than simply insufficient CD sweeps.

% Additional tables transcribed from the author rebuttal; source IDs retained for review.

\FloatBarrier

\subsection{Calibration-resampling robustness}

\label{app:calibration_resampling}

We independently sample three C4 calibration sets, each containing 128 sequences of length 2048. Within each draw, all methods use the same calibration data, preprocessing, evaluation tokens, and quantization protocol. We report the mean and sample standard deviation of final WikiText-2 and C4 perplexity. The scalar codebook uses group size 128 and AWQ+RHT preprocessing.

% Source: YjGm reply p8JeaH44v5, table AX 367

\begin{table}[htbp]
\centering
\small
\caption{Perplexity of 4-bit scalar quantization over three calibration draws (mean $\pm$ sample SD).}
\label{tab:resampling_4bit}
\begin{tabular}{llrr}
\toprule
Model & Method & WikiText-2 & C4 \\
\midrule
Llama-3.2-1B & GPTQ+CD3 & $10.625 \pm 0.028$ & $13.856 \pm 0.017$ \\
 & QuantEase & $10.474 \pm 0.014$ & $13.827 \pm 0.033$ \\
 & OJBKQ & $10.456 \pm 0.023$ & $13.672 \pm 0.022$ \\
 & MPQ & $10.392 \pm 0.010$ & $13.629 \pm 0.017$ \\
\midrule
Llama-3.2-3B & GPTQ+CD3 & $8.264 \pm 0.013$ & $11.109 \pm 0.010$ \\
 & QuantEase & $8.202 \pm 0.006$ & $11.038 \pm 0.014$ \\
 & OJBKQ & $8.176 \pm 0.014$ & $11.016 \pm 0.019$ \\
 & MPQ & $8.128 \pm 0.010$ & $10.961 \pm 0.001$ \\
\midrule
Llama-3-8B & GPTQ+CD3 & $6.557 \pm 0.011$ & $9.427 \pm 0.010$ \\
 & QuantEase & $6.507 \pm 0.024$ & $9.369 \pm 0.012$ \\
 & OJBKQ & $6.495 \pm 0.001$ & $9.364 \pm 0.013$ \\
 & MPQ & $6.464 \pm 0.004$ & $9.326 \pm 0.007$ \\
\midrule
Llama-2-13B & GPTQ+CD3 & $5.005 \pm 0.007$ & $6.575 \pm 0.001$ \\
 & QuantEase & $4.987 \pm 0.008$ & $6.573 \pm 0.002$ \\
 & OJBKQ & $4.976 \pm 0.007$ & $6.576 \pm 0.002$ \\
 & MPQ & $4.969 \pm 0.010$ & $6.553 \pm 0.002$ \\
\bottomrule
\end{tabular}
\end{table}

\FloatBarrier

% Source: YjGm reply p8JeaH44v5, table AX 465

\begin{table}[htbp]
\centering

\small
\caption{Perplexity of 3-bit scalar quantization over three calibration draws (mean $\pm$ sample SD).}
\label{tab:resampling_3bit}
\begin{tabular}{llrr}
\toprule
Model & Method & WikiText-2 & C4 \\
\midrule
Llama-3.2-1B & GPTQ+CD3 & $15.352 \pm 0.202$ & $18.950 \pm 0.191$ \\
 & QuantEase & $14.909 \pm 0.174$ & $18.759 \pm 0.178$ \\
 & OJBKQ & $14.471 \pm 0.260$ & $18.346 \pm 0.395$ \\
 & MPQ & $13.776 \pm 0.091$ & $17.502 \pm 0.036$ \\
\midrule
Llama-3.2-3B & GPTQ+CD3 & $10.575 \pm 0.042$ & $13.767 \pm 0.029$ \\
 & QuantEase & $10.098 \pm 0.135$ & $13.467 \pm 0.079$ \\
 & OJBKQ & $10.003 \pm 0.090$ & $13.266 \pm 0.065$ \\
 & MPQ & $9.604 \pm 0.014$ & $12.919 \pm 0.032$ \\
\midrule
Llama-3-8B & GPTQ+CD3 & $8.429 \pm 0.019$ & $11.559 \pm 0.032$ \\
 & QuantEase & $8.185 \pm 0.097$ & $11.327 \pm 0.016$ \\
 & OJBKQ & $8.046 \pm 0.043$ & $11.073 \pm 0.029$ \\
 & MPQ & $7.848 \pm 0.026$ & $10.855 \pm 0.010$ \\
\midrule
Llama-2-13B & GPTQ+CD3 & $5.463 \pm 0.024$ & $7.110 \pm 0.001$ \\
 & QuantEase & $5.372 \pm 0.015$ & $6.980 \pm 0.003$ \\
 & OJBKQ & $5.384 \pm 0.021$ & $7.005 \pm 0.004$ \\
 & MPQ & $5.311 \pm 0.006$ & $6.902 \pm 0.003$ \\
\bottomrule
\end{tabular}
\end{table}

\FloatBarrier

MPQ has the lowest mean perplexity across all settings, with larger gains at 3 bits.

\subsection{D4 codebook and matched refinement}

\label{app:d4_controls}

We compare LDLQ, OJBKQ ($K=5$), and MPQ using the four-dimensional 2-bit D4 codebook and identical exact block-CD updates. Tables~\ref{tab:d4_objective} and~\ref{tab:d4_perplexity} give the objective and held-out perplexity.

% AUTHOR CHECK: confirm the model and calibration-run identifier for D4 against the original logs before final approval.

% Source: jjqr follow-up YiBNrSdp30, AX tables 1361 and 1400.

\begin{table}[htbp]
\centering

\small
\caption{Calibration objective under matched D4 refinement. CD0 is the initial hard assignment.}
\label{tab:d4_objective}
\begin{tabular}{lrrr}
\toprule
Initialization & CD0 & CD1 & CD3 \\
\midrule
LDLQ & 359.869170 & 347.919460 & 340.203384 \\
OJBKQ ($K=5$) & 357.7222 & 348.1638 & 340.3826 \\
MPQ & 371.615070 & 335.443011 & 324.342445 \\
\bottomrule
\end{tabular}
\end{table}

\begin{table}[htbp]
\centering

\small
\caption{Held-out perplexity with the 2-bit D4 codebook.}
\label{tab:d4_perplexity}
\begin{tabular}{lrr}
\toprule
Method & WikiText-2 & C4 \\
\midrule
LDLQ & 20.2967 & 24.1372 \\
LDLQ + CD3 & 14.6389 & 17.6817 \\
OJBKQ ($K=5$) + CD3 & 15.0872 & 17.8354 \\
MPQ + CD3 & 12.9658 & 15.7847 \\
\bottomrule
\end{tabular}
\end{table}

MPQ starts with the highest objective and has the lowest after the first CD sweep. After three sweeps, it also has the lowest WikiText-2 and C4 perplexity.

\FloatBarrier

\subsection{Matched solver budgets and extended refinement}

\label{app:matched_budget}

Table~\ref{tab:matched_solver_budget} compares MPQ and LDLQ at similar solver-time budgets on Llama-3-8B with E8P12. Timings follow the protocol in Section~\ref{sec:solver_scalability}.

% Source: jjqr follow-up YiBNrSdp30, AX table 1444.

\begin{table}[htbp]
\centering
\small
\caption{Solver time and WikiText-2 perplexity for two offline compute budgets.}
\label{tab:matched_solver_budget}
\begin{tabular}{lrr}
\toprule
Method & Time (s) & WikiText-2 \\
\midrule
LDLQ + CD3 & 2076 & 12.756 \\
MPQ ($T_{\mathrm{vamp}}=8$) + CD3 & 2134 & 12.370 \\
LDLQ + CD6 & 3057 & 12.728 \\
MPQ ($T_{\mathrm{vamp}}=16$) + CD3 & 2951 & 11.630 \\
\bottomrule
\end{tabular}
\end{table}

At the higher budget, MPQ achieves lower perplexity than LDLQ with slightly less solver time. Its advantage over LDLQ persists at 20 CD sweeps (Table~\ref{tab:extended_cd_controls}).

% Source: jjqr reply 8GjL0oGV6Z, AX table 854; retain only unambiguous CD3 and CD20 rows.

\begin{table}[htbp]
\centering
\small
\caption{Perplexity after extended CD refinement on Llama-3-8B with E8P12.}
\label{tab:extended_cd_controls}
\begin{tabular}{lrrr}
\toprule
Initialization & CD sweeps & WikiText-2 & C4 \\
\midrule
LDLQ & 3 & 12.756 & 16.034 \\
MPQ & 3 & 11.753 & 14.818 \\
LDLQ & 20 & 12.651 & 15.943 \\
MPQ & 20 & 11.669 & 14.728 \\
\bottomrule
\end{tabular}
\end{table}

\FloatBarrier

\section{Implementation}
\label{app:implementation}

This appendix collects the exact update formulas and module identities referenced in Section~\ref{sec:method}, and concludes with a Python-style pseudocode listing of our reference implementation.

\subsection{Extrinsic-message updates, initialization, and \texorpdfstring{$\beta$}{beta} update}
\label{app:impl_extrinsic}

\paragraph{Extrinsic updates.}
Let $(\hat x_1^t,\alpha_1^t)$ denote the denoiser output and sensitivity. The outgoing precision and mean are
\begin{equation}
\eta_1^t=\gamma_1^t/\alpha_1^t,
\qquad
\gamma_2^t=\eta_1^t-\gamma_1^t,
\qquad
r_2^t=\frac{\eta_1^t\hat x_1^t-\gamma_1^t r_1^t}{\gamma_2^t}.
\label{eq:extrinsic_12}
\end{equation}
Similarly, given the linear-module output $(\hat x_2^t,\alpha_2^t)$,
\begin{equation}
\eta_2^t=\gamma_2^t/\alpha_2^t,
\qquad
\gamma_1^{t+1}=\eta_2^t-\gamma_2^t,
\qquad
r_1^{t+1}=\frac{\eta_2^t\hat x_2^t-\gamma_2^t r_2^t}{\gamma_1^{t+1}}.
\label{eq:extrinsic_21}
\end{equation}
After each extrinsic update we damp both means and precisions geometrically with a single global damping coefficient $\rho_{\mathrm{damp}}\in(0,1]$ and clip precisions to a fixed positive interval $[\gamma_{\min},\gamma_{\max}]$ to avoid degenerate messages.

\paragraph{Initialization.}
We initialize $r_1^0\!\leftarrow\!w_{\mathrm{rtn}}$, where $w_{\mathrm{rtn}}$ denotes the round-to-nearest rounding of $w$, and set the initial precision per row by
\begin{equation}
\gamma_1^0=\frac{1}{\mathrm{MSE}(w,w_{\mathrm{rtn}})},
\qquad
\mathrm{MSE}(w,w_{\mathrm{rtn}}):=\frac{1}{n}\|w-w_{\mathrm{rtn}}\|_2^2,
\label{eq:gamma1_init}
\end{equation}
clipped below by a small numerical floor. This per-row scaling adapts $\gamma_1^0$ to the difficulty of the row and removes one continuous hyperparameter. The observation-noise precision is initialized by $\beta^0=1/\overline{\mathrm{MSE}}$ with $\overline{\mathrm{MSE}}$ averaged over the row batch.

\paragraph{Online M-step for $\beta$.}
We treat the inverse temperature $\beta$ in the Gibbs measure of Eq.~\eqref{eq:posterior} as the observation-noise precision of the linear module (Eq.~\eqref{eq:linear_module_app}) and update it after each linear-module step by the standard Gaussian-likelihood M-step under that posterior. In rank space, with $u_\ell:=\tilde y_\ell - s_\ell v_\ell^\top r_2^t$ the cavity residual along the $\ell$-th singular direction and $d_\ell:=\beta^t s_\ell^2+\gamma_2^t$, the update is
\begin{equation}
\frac{1}{\beta^{t+1}}\;\leftarrow\;\frac{1}{r}\sum_{\ell=1}^{r}\!\left[\left(\frac{\gamma_2^t}{d_\ell}\right)^{\!2}u_\ell^2+\frac{s_\ell^2}{d_\ell}\right],
\label{eq:beta_em}
\end{equation}
averaged over the row batch when batched, then lightly damped and clipped to a fixed numerical range. The two summands in Eq.~\eqref{eq:beta_em} are the squared posterior-mean residual and the posterior-trace correction, respectively. Removing $\beta$ as a free hyperparameter gave more stable behavior than fixing it to a constant; the main tables therefore use $\rho_{\mathrm{damp}}$ as the only continuous solver hyperparameter.

\subsection{Exact denoiser sensitivity}
\label{app:impl_denoiser}
For one block, define
\begin{equation}
q(c\mid r,\gamma)
=
\frac{\rho(c)\exp\!\left(-\frac{\gamma}{2}\|c-r\|_2^2\right)}{\sum_{c'\in\mathcal C}\rho(c')\exp\!\left(-\frac{\gamma}{2}\|c'-r\|_2^2\right)}.
\label{eq:disc_post_app}
\end{equation}
The mean and covariance are
\begin{equation}
\hat x(r)=\sum_c c\,q(c\mid r,\gamma),
\qquad
\Sigma(r)=\sum_c (c-\hat x(r))(c-\hat x(r))^\top q(c\mid r,\gamma).
\label{eq:disc_moments_app}
\end{equation}

\begin{proposition}[Discrete denoiser Jacobian]
\label{prop:denoiser_jacobian}
For every finite codebook and finite $\gamma>0$,
\begin{equation}
\frac{\partial \hat x(r)}{\partial r}=\gamma\,\Sigma(r).
\label{eq:disc_jacobian_app}
\end{equation}
Consequently, the average sensitivity used by MPQ is exactly
\begin{equation}
\alpha_1^t=\frac{\gamma_1^t}{n}\sum_{b=1}^B \operatorname{tr}\,\Sigma_{1,b}^t.
\label{eq:disc_alpha_app}
\end{equation}
\end{proposition}

\begin{proof}
Differentiate the normalized discrete posterior. Writing $q_c:=q(c\mid r,\gamma)$ and $\hat x:=\hat x(r)$,
\begin{equation}
\frac{\partial q_c}{\partial r}=\gamma q_c(c-\hat x).
\end{equation}
Therefore
\begin{align}
\frac{\partial \hat x}{\partial r}
&=\sum_c c\left(\frac{\partial q_c}{\partial r}\right)^\top
=\gamma\sum_c q_c\,c(c-\hat x)^\top \\
&=\gamma\sum_c q_c\,(c-\hat x)(c-\hat x)^\top + \gamma\,\hat x\!\left(\sum_c q_c(c-\hat x)\right)^{\!\top} \\
&=\gamma\sum_c q_c(c-\hat x)(c-\hat x)^\top
=\gamma\,\Sigma(r),
\end{align}
where the second line writes $c=(c-\hat x)+\hat x$ inside the sum, and the third line uses $\sum_c q_c(c-\hat x)=0$. Taking traces blockwise and averaging yields Eq.~\eqref{eq:disc_alpha_app}.
\end{proof}

\subsection{Exact linear-module sensitivity}
\label{app:impl_linear}
\begin{proposition}[Linear sensitivity identity]
\label{prop:linear_sensitivity}
For the linear module
\begin{equation}
\hat x_2=(\beta H+\gamma_2 I)^{-1}(\beta g+\gamma_2 r_2),
\label{eq:linear_module_app}
\end{equation}
its Jacobian with respect to $r_2$ is
\begin{equation}
\frac{\partial \hat x_2}{\partial r_2}=\gamma_2(\beta H+\gamma_2 I)^{-1},
\label{eq:linear_jacobian_app}
\end{equation}
and hence the average sensitivity is
\begin{equation}
\alpha_2=\frac{1}{n}\operatorname{tr}\!\left[\gamma_2(\beta H+\gamma_2 I)^{-1}\right].
\label{eq:linear_alpha_app}
\end{equation}
\end{proposition}

\begin{proof}
The map $r_2\mapsto \hat x_2$ is affine, so its Jacobian is the coefficient of $r_2$, namely $\gamma_2(\beta H+\gamma_2 I)^{-1}$. Averaging the diagonal entries gives Eq.~\eqref{eq:linear_alpha_app}.
\end{proof}

\subsection{Block coordinate descent search as Mahalanobis nearest-codeword search}
\label{app:impl_mahalanobis}
Let $u=\tilde y-Tx$ be the current residual. For block $b$, define the local Gram matrix $G_b:=T_b^\top T_b$ and the local score $h_b:=T_b^\top u$.

\begin{proposition}[Equivalent blockwise Mahalanobis search]
\label{prop:mahalanobis}
Up to an additive constant independent of $c$, the block objective is
\begin{equation}
F_b(c)=\frac12(c-\mu_b)^\top G_b(c-\mu_b)+\mathrm{const},
\qquad
\mu_b:=x_b+G_b^{-1}h_b,
\label{eq:mahalanobis_obj_app}
\end{equation}
whenever $G_b$ is nonsingular. Thus exact block search is equivalent to nearest-codeword search in the Mahalanobis metric induced by $G_b$. If $G_b$ is singular, the algorithm still minimizes the quadratic form in Eq.~\eqref{eq:block_obj} directly. Note that $h_b=T_b^\top u\in\operatorname{range}(T_b^\top)=\operatorname{range}(T_b^\top T_b)=\operatorname{range}(G_b)$, so $G_b^+ h_b$ is well defined; equivalently, one may use the Moore--Penrose pseudoinverse with $\mu_b:=x_b+G_b^+h_b$ and the induced positive-semidefinite seminorm.
\end{proposition}

\begin{proof}
Expand Eq.~\eqref{eq:block_obj}:
\begin{align}
F_b(c)
&=\frac12\|u\|_2^2-u^\top T_b(c-x_b)+\frac12(c-x_b)^\top G_b(c-x_b) \\
&=\frac12(c-x_b)^\top G_b(c-x_b)-h_b^\top(c-x_b)+\mathrm{const}.
\end{align}
Completing the square around $\mu_b=x_b+G_b^{-1}h_b$ yields Eq.~\eqref{eq:mahalanobis_obj_app}.
\end{proof}

Proposition~\ref{prop:mahalanobis} makes explicit why initialization still matters even though the second stage is exact within each block. The search metric itself depends on the current residual, so the later refinement remains local.

\subsection{Reference implementation}
\label{app:impl_pseudocode}

The pseudocode below gives Python-style pseudocode for one row block of MPQ. Compared with Algorithm~\ref{alg:mpq}, it makes explicit the precision retry/clipping safeguards on $\alpha_1$ and $\gamma_2$, and a best-iterate selection rule (initialized to the RTN baseline) that passes the hard MAP projection with the smallest rank-space surrogate error to CD initialization.

\begin{verbatim}
def mpq_solve_row_batch(W, U, Sigma_A, V, codebooks, sigma, tau,
                        T_vamp, S_CD, damping):
    # ---- Shape definitions -----------------------------------------
    # W: (nrow, n) full-precision rows (nrow row-batch size, n features)
    # U, Sigma_A, V: compact SVD of A = U * diag(Sigma_A) * V^T,
    #   with Sigma_A of shape (r,) for r = rank(A)
    # codebooks[b]: (M_b, d) codeword matrix for block b
    # sigma, tau: parameters of rho_b in Eq. (2)
    nrow, n    = W.shape
    r          = Sigma_A.shape[0]
    d          = codebooks[0].shape[1]
    num_blocks = n // d                        # equals B in the paper
    # b_slice(b) := slice(b*d, (b+1)*d)        # zero-based block slice

    # ---- Stage I: finite-temperature VAMP in rank space ----
    x_rtn = round_to_nearest(W, codebooks)
    y_tilde = W @ V * Sigma_A   # (nrow, r); row j equals U^T A w_j
    s = Sigma_A # (r,)
    r1 = x_rtn.copy()                                 # (nrow, n)
    mse_init = ((W - x_rtn) ** 2).mean(axis=1)        # (nrow,)
    gamma1 = 1 / mse_init.clip(min=1e-12)            # (nrow,) per-row init
    beta = 1.0 / mse_init.mean().clip(min=1e-12)      # scalar EM init
    r2_prev, gamma2_prev = None, None

    for t in range(T_vamp):
        # ---- Denoiser (Eqs. (4)-(6)); retry gamma1 on degenerate alpha1 ----
        for _ in range(MAX_RETRIES):
            q, x1, Sigma1 = denoiser_posterior(
                r1, gamma1, codebooks,
                prior_center=W, sigma=sigma, tau=tau)
            alpha1 = gamma1 * trace_block_sum(Sigma1) / n      # Eq. (6)
            if alpha_in_safe_range(alpha1):
                break
            gamma1 = retry_rescale(gamma1, alpha1)             # multiplicative
        alpha1 = clip_alpha(alpha1)

        # ---- Extrinsic 1 -> 2 (Eqs. (19)-(20)); damping, clipping ----
        eta1   = gamma1 / alpha1
        gamma2 = clip_precision(eta1 - gamma1)
        r2     = (eta1 * x1 - gamma1 * r1) / gamma2
        if r2_prev is not None and damping < 1.0:
            gamma2 = log_damp(gamma2_prev, gamma2, damping)
            r2     = (1 - damping) * r2_prev + damping * r2

        # ---- Linear module in rank space (Eqs. (7)-(8)) ----            
        denom      = beta * s ** 2 + gamma2[:, None]           # (nrow, r)
        residual_u = y_tilde - (r2 @ V) * s                    # (nrow, r)
        x2         = r2 + (residual_u * (beta * s) / denom) @ V.T
        alpha2_raw = ((1 - r / n) +
                      (gamma2 / n) * (1.0 / denom).sum(axis=1))
        alpha2     = clip_alpha(alpha2_raw)

        # ---- EM update of beta (App. B.1) ----
        beta = em_update_beta(residual_u, gamma2, s, beta,
                              num_obs=r, damping=BETA_EM_DAMPING)

        # ---- Extrinsic 2 -> 1 (Eqs. (21)-(22)); damping ----
        eta2       = gamma2 / alpha2
        gamma1_new = clip_precision(eta2 - gamma2)
        r1_new     = (eta2 * x2 - gamma2 * r2) / gamma1_new
        if damping < 1.0:
            gamma1 = log_damp(gamma1, gamma1_new, damping)
            r1     = (1 - damping) * r1 + damping * r1_new
        else:
            gamma1, r1 = gamma1_new, r1_new
        r2_prev, gamma2_prev = r2, gamma2

    # ---- MAP projection ----
    q_final, _, _ = denoiser_posterior(
        r1, gamma1, codebooks,
        prior_center=W, sigma=sigma, tau=tau
    )
    x = blockwise_argmax(q_final, codebooks).copy()

    # ---- Stage II: exact block coordinate descent (Eqs. (11)-(12)) ----
    # In rank space T = diag(Sigma_A) . V^T;
    # column block b is s[:, None] * V[b_slice(b)].T.
    u = y_tilde - (x @ V) * s                                  # residual
    for sweep in range(S_CD):
        changed = False
        for b in range(num_blocks):
            x_b_old = x[:, b_slice(b)]
            T_b     = s[:, None] * V[b_slice(b)].T             # (r, d)
            x_b_new = argmin_F_b(u, T_b=T_b,
                                 x_b_old=x_b_old,
                                 codebook=codebooks[b])        # Eq. (12)
            if not equal(x_b_new, x_b_old):
                u -= (x_b_new - x_b_old) @ T_b.T               # update residual
                x[:, b_slice(b)] = x_b_new
                changed = True
        if not changed:
            break
    return x
\end{verbatim}

\paragraph{Helper contracts.}
The pseudocode above uses a small set of helper routines. Their contracts are as follows; all are deterministic, and their references make the listing self-contained.
\begin{itemize}
\item \texttt{round\_to\_nearest(W, codebooks)} returns, for each row of $W$ and each block $b$, the codeword in $\mathcal C_b$ closest in $\ell_2$ to the corresponding block of the row.
\item \texttt{denoiser\_posterior(r1, gamma1, codebooks, prior\_center=W,\\ sigma, tau)} evaluates Eqs.~\eqref{eq:disc_post_app}--\eqref{eq:disc_moments_app} blockwise with $\rho_b$ defined as in Eq.~\eqref{eq:posterior}; it returns the blockwise posterior $q$, mean $\hat x_1$, and covariance $\Sigma_1$.
\item \texttt{trace\_block\_sum(Sigma1)} computes $\sum_{b=1}^B\operatorname{tr}\Sigma_{1,b}$, a row-wise scalar.
\item \texttt{blockwise\_argmax(q, codebooks)} returns $\hat x$ with $\hat x_b=\arg\max_{c\in\mathcal C_b}q_b(c)$, used for MAP projection (Eq.~\eqref{eq:map_init}) and for the best-iterate diagnostic.
\item \texttt{argmin\_F\_b(u, T\_b, x\_b\_old, codebook)} solves Eq.~\eqref{eq:block_update} by enumeration over \texttt{codebook}, given the current rank-space residual $u$ and the rank-space block operator $T_b$.
\item \texttt{em\_update\_beta(residual\_u, gamma2, s, beta, num\_obs)} returns the EM-style update of $\beta$ described in Eq.~\eqref{eq:beta_em}.

\item \texttt{clip\_alpha(a)} clips into $[\varepsilon,1-\varepsilon]$ for a small fixed $\varepsilon$; \texttt{clip\_precision(g)} clips into $[\gamma_{\min},\gamma_{\max}]$ for fixed positive constants.
\item \texttt{alpha\_in\_safe\_range(alpha1)} tests whether $\alpha_1\in[\varepsilon,1-\varepsilon]$ for the same $\varepsilon$.
\item \texttt{retry\_rescale(gamma1, alpha1)} multiplies $\gamma_1$ by a fixed factor $>1$ when $\alpha_1\ge 1-\varepsilon$ and by its reciprocal when $\alpha_1\le\varepsilon$, leaving $\gamma_1$ otherwise unchanged.
\item \texttt{log\_damp(prev, new, damping)} returns $\exp\bigl((1-\mathtt{damping})\log\mathtt{prev}+\mathtt{damping}\log\mathtt{new}\bigr)$, the geometric (log-space) damping used for precisions.
\item \texttt{b\_slice(b)} returns the (Python, zero-based) index slice $[bd,(b+1)d)$ for the $b$-th block of width $d$; \texttt{equal} returns true when its two arguments are identical, where blocks may be represented either as codeword vectors or by their stored integer codeword indices.
\end{itemize}

\paragraph{Implementation notes.}
First, the per-row choice of $\gamma_1^0$ via $1/\mathrm{MSE}(w,w_{\mathrm{rtn}})$ removes one continuous hyperparameter. State evolution is locally insensitive to this choice, but per-row scaling avoids pathologically small initial precisions on easy rows.
Second, the precision retries on $\alpha_1$ and the precision clipping on $\gamma_2,\gamma_1^{t+1}$ guard against the rare degenerate cavities that finite discrete-denoiser sensitivities can produce. These safeguards activate on a small fraction of iterations in our runs and do not change the qualitative behavior of the recursion.
%Third, best-iterate selection scores each hard MAP projection $\hat x_t = \arg\max_c q_b^t(c)$ produced during the VAMP iterations against $w_{\mathrm{rtn}}$ and keeps the iterate with the smallest rank-space surrogate error. This iterate is then used as the CD initialization. By construction the surrogate error of this initialization is bounded by that of $w_{\mathrm{rtn}}$, so the Stage-I init is never worse than RTN.

\section{Limitations}
\label{app:limitations}

% We highlight two limitations of the present work.

\paragraph{Asymptotic theory.} The analysis in Section~\ref{sec:theory} takes the unanchored $\tau\to\infty$ limit and adopts an idealized right-orthogonally invariant surrogate. The MPQ algorithm itself runs at finite $\tau$ on Hessians that are RHT-conditioned but not exactly right-orthogonally invariant in the asymptotic sense. We therefore treat Theorem~\ref{thm:warmstart} and Proposition~\ref{thm:linear_response} as qualitative guidance rather than tight characterizations of the deployed algorithm. A finite-$\tau$ and finite-dimensional analysis remains open.

\paragraph{Codebook scope.} Our experiments cover uniform scalar codebooks at 3 and 4 bits and E8P12 and D4 lattice codebooks at 2 bits. Adapting MPQ to learned vector codebooks (AQLM) or trellis-coded codebooks (QTIP) is straightforward in principle, since the message-passing view in Section~\ref{sec:preliminary} is codebook-agnostic, but it requires codebook-specific extensions to the discrete denoiser that we do not evaluate here.

\section{Broader impact}
\label{app:broader_impact}

This work proposes an improved solver-side algorithm for post-training weight quantization of large language models. Our contribution sits within the foundational research category, and we view its broader societal effects as inheriting from those of the underlying language models rather than introducing new categories of risk.

Potential positive impacts include lower memory and compute requirements for deploying large language models. This can broaden access to language modeling technology in resource-constrained settings, including smaller research labs, edge devices, and regions with limited cloud infrastructure, and can reduce the energy footprint of inference at scale.

Potential negative impacts derive from the same lowering of deployment cost. Quantization techniques that make large language models cheaper to serve can also accelerate the deployment of language models in applications with harmful intent, such as disinformation and automated abuse. We do not see a direct or unique path from MPQ specifically to such harms beyond what is shared by the broader class of weight quantization methods and by language model deployment in general. The present submission releases an offline quantization solver rather than pre-trained models, fine-tuned weights, or training data.

\section{Detailed proofs}
\label{app:theory}

This appendix provides full proofs for the two main theorems stated in Section~\ref{sec:theory} and records the supporting results. Exact implementation identities used below are collected in Appendix~\ref{app:implementation}.

\subsection{Auxiliary specialization of the right-orthogonally invariant (ROI) state evolution}
\label{app:proof_se_specialization}
We restate the right-orthogonally invariant surrogate here for completeness. We analyze one row-wise PTQ subproblem through
\begin{equation}
y = A x^\star + \xi,
\qquad
\xi \sim \mathcal N(0,\beta_0^{-1} I_m),
\label{eq:app_surrogate_two}
\end{equation}
where $x^\star=[C_1^\star;\ldots;C_B^\star]$ has fixed block size $d$ and i.i.d.\ blocks from the uniform prior $\pi=\operatorname{Unif}(\mathcal C)$ on $\mathcal C\subset\mathbb R^d$. This is the unanchored $\tau\to\infty$ limit of the reference measure in Eq.~\eqref{eq:posterior}; finite $\tau$ in the algorithm is treated as an anchoring regularizer rather than as part of this surrogate. The matrix $A\in\mathbb R^{m\times n}$ is assumed ROI with limiting spectral law $\nu$ for $H=A^\top A$. For a precision $\gamma>0$, the block-discrete denoiser is
\begin{equation}
g_d(r;\gamma)=\sum_{c\in\mathcal C} c\,q(c\mid r,\gamma),
\qquad
q(c\mid r,\gamma) \propto \exp\!\left(-\frac{\gamma}{2}\|r-c\|_2^2\right),
\label{eq:app_gd}
\end{equation}
and the denoiser and linear error maps are
\begin{equation}
E_d(\gamma)=\frac{1}{d}\,\mathbb E\Bigl[\|g_d(C+\gamma^{-1/2}G;\gamma)-C\|_2^2\Bigr],
\qquad
E_{\mathrm{lin}}(\gamma)=\int \frac{1}{\beta_0\lambda+\gamma}\,d\nu(\lambda),
\label{eq:app_error_maps_two}
\end{equation}
with $C\sim\operatorname{Unif}(\mathcal C)$ and $G\sim\mathcal N(0,I_d)$. The integral expression for $E_{\mathrm{lin}}$ is the large-system counterpart of the finite-$n$ trace in Eq.~\eqref{eq:Elin_def_main}. The corresponding extrinsic-precision maps are $\Psi_d(\gamma):=E_d(\gamma)^{-1}-\gamma$ and $\Psi_{\mathrm{lin}}(\gamma):=E_{\mathrm{lin}}(\gamma)^{-1}-\gamma$.

\begin{proposition}[SE specialization to the MPQ denoiser]
\label{prop:se_specialization_appendix}
Under the surrogate model~\eqref{eq:app_surrogate_two} and the standard fixed-size block-separable extension of the right-orthogonally invariant assumptions of \citep[Theorem~1]{vamp}, the message-passing stage precisions obey
\begin{equation}
\gamma_2^t=\Psi_d(\gamma_1^t),
\qquad
\gamma_1^{t+1}=\Psi_{\mathrm{lin}}(\gamma_2^t).
\label{eq:app_se_recursion_two}
\end{equation}
Equivalently, each block sees the effective AWGN channel
\begin{equation}
R_b^t=C_b^\star+(\gamma_1^t)^{-1/2}G_b,
\qquad G_b\sim\mathcal N(0,I_d).
\end{equation}
\end{proposition}

\begin{proof}
This follows by applying the standard fixed-size block lifting of the VAMP state-evolution result of \citep[Theorem~1]{vamp}. On the block-lifted system, the MPQ denoiser is separable across fixed-size blocks and smooth for every finite $\gamma$ because the codebook is finite. By Proposition~\ref{prop:denoiser_jacobian} in Appendix~\ref{app:implementation}, the matched denoiser sensitivity is
\begin{equation}
\alpha_1 = \frac{\gamma_1}{d}\,\mathbb E\,\operatorname{tr}\,\operatorname{Cov}(C\mid R)=\gamma_1 E_d(\gamma_1).
\label{eq:app_alpha1_two}
\end{equation}
Likewise, Proposition~\ref{prop:linear_sensitivity} yields the matched linear sensitivity
\begin{equation}
\alpha_2 = \gamma_2 E_{\mathrm{lin}}(\gamma_2).
\label{eq:app_alpha2_two}
\end{equation}
Substituting Eqs.~\eqref{eq:app_alpha1_two}--\eqref{eq:app_alpha2_two} into the standard VAMP extrinsic precision updates,
\begin{equation}
\eta_i = \gamma_i/\alpha_i,
\qquad
\bar\gamma_i = \eta_i-\gamma_i,
\label{eq:app_extrinsic_two}
\end{equation}
produces
\begin{equation}
\bar\gamma_1 = E_d(\gamma_1)^{-1}-\gamma_1,
\qquad
\bar\gamma_2 = E_{\mathrm{lin}}(\gamma_2)^{-1}-\gamma_2,
\end{equation}
which is exactly Eq.~\eqref{eq:app_se_recursion_two}. The effective-channel representation follows from the scalar-equivalent model delivered by the VAMP theorem, and local convergence holds when the recursion is initialized in the basin of the selected attractive fixed point.
\end{proof}

\subsection{Proof of Theorem~\ref{thm:warmstart}}
\label{app:proof_warmstart}
\begin{proof}
For each competitor $j\neq j^\star$, under $R=C^\star+\eta$ and the uniform prior, we have
\begin{equation}
s_j(R)-s_{j^\star}(R)
=
-\Delta_j + \gamma_1^\star\langle \eta,\delta_j\rangle,
\qquad
\Delta_j:=\frac{\gamma_1^\star}{2}\|\delta_j\|_2^2 .
\label{eq:app_score_diff_two}
\end{equation}
Since $\langle \eta,\delta_j\rangle\sim \mathcal N\bigl(0,(\gamma_1^\star)^{-1}\|\delta_j\|_2^2\bigr)$, the score difference is Gaussian with mean $-\Delta_j$ and variance $\gamma_1^\star\|\delta_j\|_2^2$. Therefore
\begin{equation}
\mathbb P\bigl(s_j(R)\ge s_{j^\star}(R)\bigr)
=
Q_{\mathrm G}\!\left(\frac{\sqrt{\gamma_1^\star}\|\delta_j\|_2}{2}\right),
\label{eq:app_pairwise_tail_two}
\end{equation}
where $Q_{\mathrm G}(t)=\mathbb P(Z\ge t)$ for $Z\sim\mathcal N(0,1)$. Applying the Gaussian tail bound $Q_{\mathrm G}(t)\le e^{-t^2/2}$ for $t\ge 0$ (the argument $\sqrt{\gamma_1^\star}\|\delta_j\|_2/2$ is nonnegative since $\delta_j\neq 0$ for $j\neq j^\star$ and $\gamma_1^\star>0$) gives
\begin{equation}
\mathbb P\bigl(s_j(R)\ge s_{j^\star}(R)\bigr)
\le
\exp\!\left(-\frac{\gamma_1^\star\|\delta_j\|_2^2}{8}\right).
\end{equation}
The MAP decoder fails only if at least one competitor outranks the true codeword, so a union bound over $j\neq j^\star$ yields Eq.~\eqref{eq:map_bound_two}.

For the expected number of incorrect blocks under the i.i.d.\ block surrogate, let $\mathsf E_b$ be the event that block $b$ is decoded incorrectly by the MAP projection. Conditional on a fixed true codeword, the same bound applies to each block, and summing over $b$ gives the factor $B$. Averaging the conditional bound over the uniform draw of $C_b^\star$ gives the unconditional bound.
\end{proof}

\subsection{Proof of Proposition~\ref{thm:linear_response}}
\label{app:proof_linear_response}
We first record the orthogonal-similarity invariance that justifies analyzing the unrotated path $H_t=D_tH_0D_t$ in the AWQ+RHT pipeline.

\begin{lemma}[Orthogonal invariance of the linear module]
\label{lem:orthogonal_invariance_elin}
For any orthogonal matrix $O$ and any symmetric $H\succeq 0$,
\begin{equation}
\frac1n\operatorname{tr}\!\left[(\beta_0 O^\top H O+\gamma I)^{-1}\right]
=
\frac1n\operatorname{tr}\!\left[(\beta_0 H+\gamma I)^{-1}\right].
\label{eq:orth_invariance_elin}
\end{equation}
Consequently, if
\begin{equation}
H_{\mathrm{raw}}=O^\top H_0 O,
\qquad
H_{\mathrm{prep}}=O^\top D H_0 D O,
\end{equation}
as in the AWQ+RHT pipeline, then
\begin{equation}
E_{\mathrm{lin}}(H_{\mathrm{prep}},\gamma)-E_{\mathrm{lin}}(H_{\mathrm{raw}},\gamma)
=
E_{\mathrm{lin}}(D H_0 D,\gamma)-E_{\mathrm{lin}}(H_0,\gamma).
\end{equation}
\end{lemma}

\begin{proof}
Since $(O^\top M O)^{-1}=O^\top M^{-1}O$ for any invertible matrix $M$, we have
\begin{equation}
(\beta_0 O^\top H O+\gamma I)^{-1}
=
O^\top (\beta_0 H+\gamma I)^{-1}O.
\end{equation}
Taking traces and using cyclicity gives Eq.~\eqref{eq:orth_invariance_elin}.
\end{proof}

\begin{proof}[Proof of Theorem~\ref{thm:linear_response}]
By Lemma~\ref{lem:orthogonal_invariance_elin}, the AWQ+RHT comparison reduces to the unrotated path $H_t=D_tH_0D_t$, which we now analyze. Define
\begin{equation}
Q_t(\gamma):=\beta_0 H_t + \gamma I,
\qquad
E_{\mathrm{lin}}(H_t,\gamma)=\frac1n\,\operatorname{tr}\bigl[Q_t(\gamma)^{-1}\bigr].
\end{equation}
Since $D_t = e^{tL}$, differentiating $H_t=D_tH_0D_t$ gives
\begin{equation}
\dot H_t = L H_t + H_t L.
\label{eq:app_Hdot_two}
\end{equation}
Hence
\begin{equation}
\dot Q_t = \beta_0(LH_t+H_tL).
\end{equation}
Using the inverse-derivative identity $\frac{d}{dt}Q_t^{-1} = -Q_t^{-1}\dot Q_t Q_t^{-1}$, we obtain
\begin{align}
\frac{d}{dt}E_{\mathrm{lin}}(H_t,\gamma)
&= \frac1n\frac{d}{dt}\operatorname{tr}\bigl[Q_t^{-1}\bigr] \\
&= -\frac1n\operatorname{tr}\bigl[Q_t^{-1}\dot Q_t Q_t^{-1}\bigr] \\
&= -\frac{\beta_0}{n}\operatorname{tr}\bigl[Q_t^{-1}(LH_t+H_tL)Q_t^{-1}\bigr].
\label{eq:app_elin_start_two}
\end{align}
Now $Q_t$ is a polynomial in $H_t$, so $Q_t^{-1}$ commutes with $H_t$; we use this commutation below. Therefore
\begin{align}
\operatorname{tr}\bigl[Q_t^{-1}LH_tQ_t^{-1}\bigr]
&= \operatorname{tr}\bigl[LH_tQ_t^{-2}\bigr], \\
\operatorname{tr}\bigl[Q_t^{-1}H_tLQ_t^{-1}\bigr]
&= \operatorname{tr}\bigl[LQ_t^{-1}H_tQ_t^{-1}\bigr]
 = \operatorname{tr}\bigl[LH_tQ_t^{-2}\bigr].
\end{align}
Substituting into Eq.~\eqref{eq:app_elin_start_two} yields Eq.~\eqref{eq:Elin_derivative_two},
\begin{equation}
\frac{d}{dt}E_{\mathrm{lin}}(H_t,\gamma)
=
-\frac{2\beta_0}{n}\operatorname{tr}\!\left[L H_t (\beta_0 H_t + \gamma I)^{-2}\right].
\end{equation}
Integrating from $t=0$ to $t=1$ gives Eq.~\eqref{eq:Elin_integral_two}.
\end{proof}

For interpretation, it is convenient to rewrite the trace coordinatewise.

\begin{corollary}[Coordinate form of the resolvent response]
\label{cor:coordinate_response_elin}
Let
\begin{equation}
\kappa_i(t,\gamma):=\left[H_t(\beta_0 H_t+\gamma I)^{-2}\right]_{ii}.
\end{equation}
Then $\kappa_i(t,\gamma)\ge 0$ for every $i,t,\gamma$, and
\begin{equation}
E_{\mathrm{lin}}(H_1,\gamma)-E_{\mathrm{lin}}(H_0,\gamma)
=
-\frac{2\beta_0}{n}\sum_{i=1}^n \ell_i\int_0^1 \kappa_i(t,\gamma)\,dt.
\label{eq:coordinate_response_elin}
\end{equation}
\end{corollary}

\begin{proof}
Because $L=\operatorname{diag}(\ell_1,\dots,\ell_n)$ is diagonal,
\begin{equation}
\operatorname{tr}\!\left[L H_t (\beta_0 H_t+\gamma I)^{-2}\right]
=
\sum_{i=1}^n \ell_i\left[H_t(\beta_0 H_t+\gamma I)^{-2}\right]_{ii},
\end{equation}
which gives Eq.~\eqref{eq:coordinate_response_elin}. To see nonnegativity, diagonalize $H_t=U_t\Lambda_t U_t^\top$ with $\Lambda_t=\operatorname{diag}(\lambda_1,\dots,\lambda_n)$, $\lambda_j\ge 0$. Then
\begin{equation}
H_t(\beta_0 H_t+\gamma I)^{-2}
=
U_t\operatorname{diag}\!\left(\frac{\lambda_j}{(\beta_0\lambda_j+\gamma)^2}\right)U_t^\top.
\end{equation}
This matrix is positive semidefinite, so every diagonal entry, in particular $\kappa_i(t,\gamma)$, is nonnegative.
\end{proof}

\section{Existing Assets}
\label{app:assets}
Datasets. 
    \begin{itemize}
        \item  WikiText-2~\cite{wikitext}, CC-BY-SA 3.0; \item C4~\cite{c4}, ODC-BY 1.0 (released through the AllenAI mirror);
        \item ARC-Challenge and ARC-Easy~\cite{arc}, CC-BY-SA 4.0;
        \item HellaSwag~\cite{hellaswag}, MIT License;
        \item PIQA~\cite{piqa}, Academic Free License v3.0;
        \item Winogrande~\cite{winogrande}, CC-BY 4.0.
    \end{itemize}

    Models.
    \begin{itemize}
        \item Llama-2-13B and Llama-2-70B~\cite{llama2}, Llama 2 Community License Agreement;
        \item Llama-3-8B, Llama-3.2-1B, and Llama-3.2-3B~\cite{llama3}, Meta Llama 3 Community License;
        \item Qwen3-8B~\cite{qwen3}, Apache License 2.0;
        \item Mistral-7B-v0.3~\cite{Mistral7}, Apache License 2.0.
    \end{itemize}
    
    Tooling. 
    \begin{itemize}
        \item lm-evaluation-harness~\cite{lm-eval}, MIT License.
        \item GPTQModel~\cite{gptqmodel}, Apache License 2.0;
        \item QuIP\#~\cite{quipsharp},  GPL-3.0 license.
    \end{itemize}

%%%%%%%%%%%%%%%%%%%%%%%%%%%%%%%%%%%%%%%%%%%%%%%%%%%%%%%%%%%%

\newpage
\section*{NeurIPS Paper Checklist}

\begin{enumerate}

\item {\bf Claims}
    \item[] Question: Do the main claims made in the abstract and introduction accurately reflect the paper's contributions and scope?
    \item[] Answer: \answerYes{} % Replace by \answerYes{}, \answerNo{}, or \answerNA{}.
    \item[] Justification: %\justificationTODO{}
    The abstract and Section 1 state the message-passing view of row-wise PTQ, the MPQ algorithm, and empirical improvements over standard hard-decision baselines. Section 6 reports results matching these claims across seven dense models and three bit budgets.
    \item[] Guidelines:
    \begin{itemize}
        \item The answer \answerNA{} means that the abstract and introduction do not include the claims made in the paper.
        \item The abstract and/or introduction should clearly state the claims made, including the contributions made in the paper and important assumptions and limitations. A \answerNo{} or \answerNA{} answer to this question will not be perceived well by the reviewers. 
        \item The claims made should match theoretical and experimental results, and reflect how much the results can be expected to generalize to other settings. 
        \item It is fine to include aspirational goals as motivation as long as it is clear that these goals are not attained by the paper. 
    \end{itemize}

\item {\bf Limitations}
    \item[] Question: Does the paper discuss the limitations of the work performed by the authors?
    \item[] Answer: \answerYes{} % Replace by \answerYes{}, \answerNo{}, or \answerNA{}.
    \item[] Justification: %\justificationTODO{}
    A dedicated Limitations section in Appendix~\ref{app:limitations} discusses the asymptotic-only theory and the codebook scope of the empirical evaluation.
    \item[] Guidelines:
    \begin{itemize}
        \item The answer \answerNA{} means that the paper has no limitation while the answer \answerNo{} means that the paper has limitations, but those are not discussed in the paper. 
        \item The authors are encouraged to create a separate ``Limitations'' section in their paper.
        \item The paper should point out any strong assumptions and how robust the results are to violations of these assumptions (e.g., independence assumptions, noiseless settings, model well-specification, asymptotic approximations only holding locally). The authors should reflect on how these assumptions might be violated in practice and what the implications would be.
        \item The authors should reflect on the scope of the claims made, e.g., if the approach was only tested on a few datasets or with a few runs. In general, empirical results often depend on implicit assumptions, which should be articulated.
        \item The authors should reflect on the factors that influence the performance of the approach. For example, a facial recognition algorithm may perform poorly when image resolution is low or images are taken in low lighting. Or a speech-to-text system might not be used reliably to provide closed captions for online lectures because it fails to handle technical jargon.
        \item The authors should discuss the computational efficiency of the proposed algorithms and how they scale with dataset size.
        \item If applicable, the authors should discuss possible limitations of their approach to address problems of privacy and fairness.
        \item While the authors might fear that complete honesty about limitations might be used by reviewers as grounds for rejection, a worse outcome might be that reviewers discover limitations that aren't acknowledged in the paper. The authors should use their best judgment and recognize that individual actions in favor of transparency play an important role in developing norms that preserve the integrity of the community. Reviewers will be specifically instructed to not penalize honesty concerning limitations.
    \end{itemize}

\item {\bf Theory assumptions and proofs}
    \item[] Question: For each theoretical result, does the paper provide the full set of assumptions and a complete (and correct) proof?
    \item[] Answer: \answerYes{} % Replace by \answerYes{}, \answerNo{}, or \answerNA{}.
    \item[] Justification: %\justificationTODO{}
    Theorem~\ref{thm:warmstart} and Proposition~\ref{thm:linear_response} state assumptions explicitly. Proof sketches appear in the main text; complete proofs and the auxiliary state-evolution specialization are in Appendix~\ref{app:theory}.
    \item[] Guidelines:
    \begin{itemize}
        \item The answer \answerNA{} means that the paper does not include theoretical results. 
        \item All the theorems, formulas, and proofs in the paper should be numbered and cross-referenced.
        \item All assumptions should be clearly stated or referenced in the statement of any theorems.
        \item The proofs can either appear in the main paper or the supplemental material, but if they appear in the supplemental material, the authors are encouraged to provide a short proof sketch to provide intuition. 
        \item Inversely, any informal proof provided in the core of the paper should be complemented by formal proofs provided in appendix or supplemental material.
        \item Theorems and Lemmas that the proof relies upon should be properly referenced. 
    \end{itemize}

    \item {\bf Experimental result reproducibility}
    \item[] Question: Does the paper fully disclose all the information needed to reproduce the main experimental results of the paper to the extent that it affects the main claims and/or conclusions of the paper (regardless of whether the code and data are provided or not)?
    \item[] Answer: \answerYes{} % Replace by \answerYes{}, \answerNo{}, or \answerNA{}.
    \item[] Justification: %\justificationTODO{}
    Section~\ref{sec:setup} specifies the calibration protocol, preprocessing pipeline, codebooks, and baseline configurations. Appendix~\ref{app:pilot} gives MPQ hyperparameters, Algorithm~\ref{alg:mpq} and Appendix~\ref{app:impl_pseudocode} provide the algorithm and a Python-style pseudocode reference.
    \item[] Guidelines:
    \begin{itemize}
        \item The answer \answerNA{} means that the paper does not include experiments.
        \item If the paper includes experiments, a \answerNo{} answer to this question will not be perceived well by the reviewers: Making the paper reproducible is important, regardless of whether the code and data are provided or not.
        \item If the contribution is a dataset and\slash or model, the authors should describe the steps taken to make their results reproducible or verifiable. 
        \item Depending on the contribution, reproducibility can be accomplished in various ways. For example, if the contribution is a novel architecture, describing the architecture fully might suffice, or if the contribution is a specific model and empirical evaluation, it may be necessary to either make it possible for others to replicate the model with the same dataset, or provide access to the model. In general. releasing code and data is often one good way to accomplish this, but reproducibility can also be provided via detailed instructions for how to replicate the results, access to a hosted model (e.g., in the case of a large language model), releasing of a model checkpoint, or other means that are appropriate to the research performed.
        \item While NeurIPS does not require releasing code, the conference does require all submissions to provide some reasonable avenue for reproducibility, which may depend on the nature of the contribution. For example
        \begin{enumerate}
            \item If the contribution is primarily a new algorithm, the paper should make it clear how to reproduce that algorithm.
            \item If the contribution is primarily a new model architecture, the paper should describe the architecture clearly and fully.
            \item If the contribution is a new model (e.g., a large language model), then there should either be a way to access this model for reproducing the results or a way to reproduce the model (e.g., with an open-source dataset or instructions for how to construct the dataset).
            \item We recognize that reproducibility may be tricky in some cases, in which case authors are welcome to describe the particular way they provide for reproducibility. In the case of closed-source models, it may be that access to the model is limited in some way (e.g., to registered users), but it should be possible for other researchers to have some path to reproducing or verifying the results.
        \end{enumerate}
    \end{itemize}

\item {\bf Open access to data and code}
    \item[] Question: Does the paper provide open access to the data and code, with sufficient instructions to faithfully reproduce the main experimental results, as described in supplemental material?
    \item[] Answer: \answerNo{} % Replace by \answerYes{}, \answerNo{}, or \answerNA{}.
    \item[] Justification: %\justificationTODO{}
    We do not release code at submission. The reference Python-style pseudocode in Appendix~\ref{app:impl_pseudocode}, the exact update formulas in Appendix~\ref{app:implementation}, and the hyperparameters in Appendix~\ref{app:pilot} together specify the algorithm sufficiently for reproduction. We commit to releasing the codebase at the camera-ready stage.
    \item[] Guidelines:
    \begin{itemize}
        \item The answer \answerNA{} means that paper does not include experiments requiring code.
        \item Please see the NeurIPS code and data submission guidelines (\url{https://neurips.cc/public/guides/CodeSubmissionPolicy}) for more details.
        \item While we encourage the release of code and data, we understand that this might not be possible, so \answerNo{} is an acceptable answer. Papers cannot be rejected simply for not including code, unless this is central to the contribution (e.g., for a new open-source benchmark).
        \item The instructions should contain the exact command and environment needed to run to reproduce the results. See the NeurIPS code and data submission guidelines (\url{https://neurips.cc/public/guides/CodeSubmissionPolicy}) for more details.
        \item The authors should provide instructions on data access and preparation, including how to access the raw data, preprocessed data, intermediate data, and generated data, etc.
        \item The authors should provide scripts to reproduce all experimental results for the new proposed method and baselines. If only a subset of experiments are reproducible, they should state which ones are omitted from the script and why.
        \item At submission time, to preserve anonymity, the authors should release anonymized versions (if applicable).
        \item Providing as much information as possible in supplemental material (appended to the paper) is recommended, but including URLs to data and code is permitted.
    \end{itemize}

\item {\bf Experimental setting/details}
    \item[] Question: Does the paper specify all the training and test details (e.g., data splits, hyperparameters, how they were chosen, type of optimizer) necessary to understand the results?
    \item[] Answer: \answerYes{} % Replace by \answerYes{}, \answerNo{}, or \answerNA{}.
    \item[] Justification: %\justificationTODO{}
    Section~\ref{sec:setup} specifies calibration data, preprocessing, codebooks, and baseline variants. The pilot sweep that motivates the hyperparameter values are reported in Figure~\ref{fig:stability_appendix} of Appendix~\ref{app:pilot}.
    \item[] Guidelines:
    \begin{itemize}
        \item The answer \answerNA{} means that the paper does not include experiments.
        \item The experimental setting should be presented in the core of the paper to a level of detail that is necessary to appreciate the results and make sense of them.
        \item The full details can be provided either with the code, in appendix, or as supplemental material.
    \end{itemize}

\item {\bf Experiment statistical significance}
    \item[] Question: Does the paper report error bars suitably and correctly defined or other appropriate information about the statistical significance of the experiments?
    \item[] Answer: \answerNo{} % Replace by \answerYes{}, \answerNo{}, or \answerNA{}.
    \item[] Justification: %\justificationTODO{}
    Appendix~\ref{app:calibration_resampling} reports mean and sample standard deviation over three paired C4 calibration draws for 3-bit and 4-bit scalar quantization on four Llama models. Error bars measure calibration-sampling variability. Other experiments report individual runs.
    \item[] Guidelines:
    \begin{itemize}
        \item The answer \answerNA{} means that the paper does not include experiments.
        \item The authors should answer \answerYes{} if the results are accompanied by error bars, confidence intervals, or statistical significance tests, at least for the experiments that support the main claims of the paper.
        \item The factors of variability that the error bars are capturing should be clearly stated (for example, train/test split, initialization, random drawing of some parameter, or overall run with given experimental conditions).
        \item The method for calculating the error bars should be explained (closed form formula, call to a library function, bootstrap, etc.)
        \item The assumptions made should be given (e.g., Normally distributed errors).
        \item It should be clear whether the error bar is the standard deviation or the standard error of the mean.
        \item It is OK to report 1-sigma error bars, but one should state it. The authors should preferably report a 2-sigma error bar than state that they have a 96\% CI, if the hypothesis of Normality of errors is not verified.
        \item For asymmetric distributions, the authors should be careful not to show in tables or figures symmetric error bars that would yield results that are out of range (e.g., negative error rates).
        \item If error bars are reported in tables or plots, the authors should explain in the text how they were calculated and reference the corresponding figures or tables in the text.
    \end{itemize}

\item {\bf Experiments compute resources}
    \item[] Question: For each experiment, does the paper provide sufficient information on the computer resources (type of compute workers, memory, time of execution) needed to reproduce the experiments?
    \item[] Answer: \answerNo{} % Replace by \answerYes{}, \answerNo{}, or \answerNA{}.
    \item[] Justification: %\justificationTODO{}
    Section~\ref{sec:experiments} specifies the timing and memory protocol on a desktop with single RTX 5090 and 64 GB RAM. The Llama-2-70B runs use H200 or L40S cluster but did not record the compute time. 
    \item[] Guidelines:
    \begin{itemize}
        \item The answer \answerNA{} means that the paper does not include experiments.
        \item The paper should indicate the type of compute workers CPU or GPU, internal cluster, or cloud provider, including relevant memory and storage.
        \item The paper should provide the amount of compute required for each of the individual experimental runs as well as estimate the total compute. 
        \item The paper should disclose whether the full research project required more compute than the experiments reported in the paper (e.g., preliminary or failed experiments that didn't make it into the paper). 
    \end{itemize}
    
\item {\bf Code of ethics}
    \item[] Question: Does the research conducted in the paper conform, in every respect, with the NeurIPS Code of Ethics \url{https://neurips.cc/public/EthicsGuidelines}?
    \item[] Answer: \answerYes{} % Replace by \answerYes{}, \answerNo{}, or \answerNA{}.
    \item[] Justification: %\justificationTODO{}
    The research conforms to the NeurIPS Code of Ethics. The paper studies algorithmic improvements to weight quantization for already-publicly-released language models and does not involve human subjects, sensitive data, or new model releases.
    \item[] Guidelines:
    \begin{itemize}
        \item The answer \answerNA{} means that the authors have not reviewed the NeurIPS Code of Ethics.
        \item If the authors answer \answerNo, they should explain the special circumstances that require a deviation from the Code of Ethics.
        \item The authors should make sure to preserve anonymity (e.g., if there is a special consideration due to laws or regulations in their jurisdiction).
    \end{itemize}

\item {\bf Broader impacts}
    \item[] Question: Does the paper discuss both potential positive societal impacts and negative societal impacts of the work performed?
    \item[] Answer: \answerYes{} % Replace by \answerYes{}, \answerNo{}, or \answerNA{}.
    \item[] Justification: %\justificationTODO{}
    Appendix~\ref{app:broader_impact} discusses both potential positive impacts (lower deployment cost, broader accessibility, reduced inference energy footprint) and potential negative impacts (acceleration of harmful language model applications) of the work.
    \item[] Guidelines:
    \begin{itemize}
        \item The answer \answerNA{} means that there is no societal impact of the work performed.
        \item If the authors answer \answerNA{} or \answerNo, they should explain why their work has no societal impact or why the paper does not address societal impact.
        \item Examples of negative societal impacts include potential malicious or unintended uses (e.g., disinformation, generating fake profiles, surveillance), fairness considerations (e.g., deployment of technologies that could make decisions that unfairly impact specific groups), privacy considerations, and security considerations.
        \item The conference expects that many papers will be foundational research and not tied to particular applications, let alone deployments. However, if there is a direct path to any negative applications, the authors should point it out. For example, it is legitimate to point out that an improvement in the quality of generative models could be used to generate Deepfakes for disinformation. On the other hand, it is not needed to point out that a generic algorithm for optimizing neural networks could enable people to train models that generate Deepfakes faster.
        \item The authors should consider possible harms that could arise when the technology is being used as intended and functioning correctly, harms that could arise when the technology is being used as intended but gives incorrect results, and harms following from (intentional or unintentional) misuse of the technology.
        \item If there are negative societal impacts, the authors could also discuss possible mitigation strategies (e.g., gated release of models, providing defenses in addition to attacks, mechanisms for monitoring misuse, mechanisms to monitor how a system learns from feedback over time, improving the efficiency and accessibility of ML).
    \end{itemize}
    
\item {\bf Safeguards}
    \item[] Question: Does the paper describe safeguards that have been put in place for responsible release of data or models that have a high risk for misuse (e.g., pre-trained language models, image generators, or scraped datasets)?
    \item[] Answer: \answerNA{} % Replace by \answerYes{}, \answerNo{}, or \answerNA{}.
    \item[] Justification: %\justificationTODO{}
    The paper does not release new pre-trained models, datasets, or other assets with high misuse risk. The contribution is an algorithm for quantizing existing publicly released models.
    \item[] Guidelines:
    \begin{itemize}
        \item The answer \answerNA{} means that the paper poses no such risks.
        \item Released models that have a high risk for misuse or dual-use should be released with necessary safeguards to allow for controlled use of the model, for example by requiring that users adhere to usage guidelines or restrictions to access the model or implementing safety filters. 
        \item Datasets that have been scraped from the Internet could pose safety risks. The authors should describe how they avoided releasing unsafe images.
        \item We recognize that providing effective safeguards is challenging, and many papers do not require this, but we encourage authors to take this into account and make a best faith effort.
    \end{itemize}

\item {\bf Licenses for existing assets}
    \item[] Question: Are the creators or original owners of assets (e.g., code, data, models), used in the paper, properly credited and are the license and terms of use explicitly mentioned and properly respected?
    \item[] Answer: \answerYes{} % Replace by \answerYes{}, \answerNo{}, or \answerNA{}.
    \item[] Justification: %\justificationTODO{}
    All datasets, models, and tooling used in our experiments are publicly available, cited in the references, and used in accordance with their respective license terms. See Appendix~\ref{app:assets}.

    \item[] Guidelines:
    \begin{itemize}
        \item The answer \answerNA{} means that the paper does not use existing assets.
        \item The authors should cite the original paper that produced the code package or dataset.
        \item The authors should state which version of the asset is used and, if possible, include a URL.
        \item The name of the license (e.g., CC-BY 4.0) should be included for each asset.
        \item For scraped data from a particular source (e.g., website), the copyright and terms of service of that source should be provided.
        \item If assets are released, the license, copyright information, and terms of use in the package should be provided. For popular datasets, \url{paperswithcode.com/datasets} has curated licenses for some datasets. Their licensing guide can help determine the license of a dataset.
        \item For existing datasets that are re-packaged, both the original license and the license of the derived asset (if it has changed) should be provided.
        \item If this information is not available online, the authors are encouraged to reach out to the asset's creators.
    \end{itemize}

\item {\bf New assets}
    \item[] Question: Are new assets introduced in the paper well documented and is the documentation provided alongside the assets?
    \item[] Answer: \answerNA{} % Replace by \answerYes{}, \answerNo{}, or \answerNA{}.
    \item[] Justification: %\justificationTODO{}
    This submission does not release new assets. The paper introduces an algorithm (MPQ); it does not release new datasets, pre-trained models, or other artifacts. We commit to releasing an open-source reference implementation of MPQ along with all baseline solver implementations, calibration scripts, and evaluation scripts at the camera-ready stage.
    \item[] Guidelines:
    \begin{itemize}
        \item The answer \answerNA{} means that the paper does not release new assets.
        \item Researchers should communicate the details of the dataset\slash code\slash model as part of their submissions via structured templates. This includes details about training, license, limitations, etc. 
        \item The paper should discuss whether and how consent was obtained from people whose asset is used.
        \item At submission time, remember to anonymize your assets (if applicable). You can either create an anonymized URL or include an anonymized zip file.
    \end{itemize}

\item {\bf Crowdsourcing and research with human subjects}
    \item[] Question: For crowdsourcing experiments and research with human subjects, does the paper include the full text of instructions given to participants and screenshots, if applicable, as well as details about compensation (if any)? 
    \item[] Answer: \answerNA{} % Replace by \answerYes{}, \answerNo{}, or \answerNA{}.
    \item[] Justification: %\justificationTODO{}
    The paper does not involve crowdsourcing or research with human subjects.
    \item[] Guidelines:
    \begin{itemize}
        \item The answer \answerNA{} means that the paper does not involve crowdsourcing nor research with human subjects.
        \item Including this information in the supplemental material is fine, but if the main contribution of the paper involves human subjects, then as much detail as possible should be included in the main paper. 
        \item According to the NeurIPS Code of Ethics, workers involved in data collection, curation, or other labor should be paid at least the minimum wage in the country of the data collector. 
    \end{itemize}

\item {\bf Institutional review board (IRB) approvals or equivalent for research with human subjects}
    \item[] Question: Does the paper describe potential risks incurred by study participants, whether such risks were disclosed to the subjects, and whether Institutional Review Board (IRB) approvals (or an equivalent approval/review based on the requirements of your country or institution) were obtained?
    \item[] Answer: \answerNA{} % Replace by \answerYes{}, \answerNo{}, or \answerNA{}.
    \item[] Justification: %\justificationTODO{}
    The paper does not involve research with human subjects, so IRB approval is not applicable.
    \item[] Guidelines:
    \begin{itemize}
        \item The answer \answerNA{} means that the paper does not involve crowdsourcing nor research with human subjects.
        \item Depending on the country in which research is conducted, IRB approval (or equivalent) may be required for any human subjects research. If you obtained IRB approval, you should clearly state this in the paper. 
        \item We recognize that the procedures for this may vary significantly between institutions and locations, and we expect authors to adhere to the NeurIPS Code of Ethics and the guidelines for their institution. 
        \item For initial submissions, do not include any information that would break anonymity (if applicable), such as the institution conducting the review.
    \end{itemize}

\item {\bf Declaration of LLM usage}
    \item[] Question: Does the paper describe the usage of LLMs if it is an important, original, or non-standard component of the core methods in this research? Note that if the LLM is used only for writing, editing, or formatting purposes and does \emph{not} impact the core methodology, scientific rigor, or originality of the research, declaration is not required.
    %this research? 
    \item[] Answer: \answerNA{} % Replace by \answerYes{}, \answerNo{}, or \answerNA{}.
    \item[] Justification: %\justificationTODO{}
    LLMs were used for writing polishing only in this paper.
    \item[] Guidelines:
    \begin{itemize}
        \item The answer \answerNA{} means that the core method development in this research does not involve LLMs as any important, original, or non-standard components.
        \item Please refer to our LLM policy in the NeurIPS handbook for what should or should not be described.
    \end{itemize}

\end{enumerate}

\end{document}